\documentclass{article}
\usepackage[utf8]{inputenc}
\usepackage[left=3cm,right=3cm,top=3cm,bottom=3cm]{geometry}

\usepackage{microtype}
\usepackage{graphicx}
\usepackage{subcaption}
\usepackage{booktabs} 
\usepackage[compress]{natbib}

\usepackage{hyperref}

\usepackage{euler}
\usepackage[T1]{fontenc}

\usepackage{amsmath}
\usepackage{amssymb}
\usepackage{mathtools}
\usepackage{amsthm}
\usepackage{nicefrac}       

\usepackage{makecell}
\usepackage{multirow}
\usepackage{enumitem}
\usepackage{stfloats}
\usepackage{algorithm}
\usepackage{algorithmic}
\usepackage{xspace}
\usepackage{pifont}
\newcommand{\xmark}{\textcolor{red}{\ding{55}}\xspace}%
\usepackage[color=gray, size=tiny, textsize=tiny]{todonotes}
\usepackage{comment}
\usepackage{alltt}
\newcommand{\cond}[1]{{\bf \textsf{C#1}}\xspace}

\newcommand{\mclass}{\mathcal{M}}
\newcommand{\loss}{\ell}
\newcommand{\Loss}{\mathcal{L}}
\newcommand{\alg}{\mathcal{A}}
\newcommand{\falg}{\mathcal{F}}
\newcommand{\halg}{\mathcal{H}}
\newcommand{\btheta}{\boldsymbol{\theta}}
\newcommand{\bphi}{\boldsymbol{\phi}}
\newcommand{\bTheta}{\boldsymbol{\Theta}}
\newcommand{\bPhi}{\boldsymbol{\Phi}}

\newcommand{\E}{\mathbb{E}}

\newcommand{\R}{\mathbb{R}}
\newcommand{\Z}{\mathbb{Z}}
\newcommand{\C}{\mathbb{C}}
\newcommand{\tsum}{\textstyle{\sum}}
\newcommand{\flora}{{\sf FLoRA}\xspace}
\newcommand{\flhpo}{{\sf FL-HPO}\xspace}

\numberwithin{equation}{section}

\theoremstyle{plain}
\newtheorem{theorem}{Theorem}[section]
\newtheorem{proposition}[theorem]{Proposition}

\newtheorem{corollary}[theorem]{Corollary}
\theoremstyle{definition}
\newtheorem{definition}[theorem]{Definition}
\newtheorem{assumption}[theorem]{Assumption}
\theoremstyle{remark}

\title{
Single-shot Hyper-parameter Optimization for\\
Federated Learning:\\ A General Algorithm \& Analysis}

\author{%
  Yi Zhou
  \qquad
  Parikshit Ram
  \qquad
  Theodoros Salonidis
  \\
  Nathalie Baracaldo
  \qquad
  Horst Samulowitz
  \qquad
  Heiko Ludwig\\
  {\sf {IBM Research}}\\
  \texttt{\{yi.zhou, Parikshit.Ram\}@ibm.com}\\
  \texttt{\{tsaloni, baracald, samulowitz, hludwig\}@us.ibm.com}
}
\begin{document}
\maketitle

\begin{abstract}
We address the relatively unexplored problem of hyper-parameter optimization (HPO) for federated learning (\flhpo). We introduce {\bf F}ederated {\bf Lo}ss Su{\bf R}face {\bf A}ggregation (\flora), a general \flhpo  solution framework that can address use cases of tabular data and any Machine Learning (ML) model including gradient boosting training algorithms and therefore further expands the scope of \flhpo. 
\flora enables single-shot \flhpo: identifying a single set of good hyper-parameters that are subsequently used in a {\em single} FL training. Thus, it enables \flhpo solutions with minimal additional communication overhead compared to FL training without HPO. We theoretically characterize the optimality gap of \flora, which explicitly accounts for the heterogeneous non-iid nature of the parties' local data distributions, a dominant characteristic of FL systems. Our empirical evaluation of \flora for multiple ML algorithms on seven OpenML datasets demonstrates significant model accuracy improvements over the considered baseline, and robustness to increasing number of parties involved in \flhpo training.
\end{abstract}

\section{Introduction} \label{sec:prob-setup}

Traditional machine learning (ML) approaches require training data to be gathered at a central location where the learning algorithm runs. In real world scenarios, however, training data is often subject to privacy or regulatory constraints restricting the way data can be shared, used and transmitted. Examples of such regulations include the European General Data Protection Regulation (GDPR), California Consumer Privacy Act (CCPA), Cybersecurity Law of China (CLA) and HIPAA, among others.
Federated learning (FL), first proposed in~\citet{mcmahan2017communication}, has recently become a popular approach to address privacy concerns by allowing collaborative training of ML models among multiple parties where each party can keep its data private.

\paragraph{\flhpo problem.} Despite the privacy protection FL brings along, there are many open problems in FL domain~\citep{kairouz2019advances,khodak2021federated}, one of which is hyper-parameter optimization for FL.
Existing FL systems require a user (or all participating parties) to pre-set (agree on) multiple hyper-parameters (HPs) (i)~for the model being trained (such as number of layers and batch size for neural networks or tree depth and number of trees in tree ensembles), and (ii)~for the the aggregator (if such hyper-parameters exist).  Hyper-parameter optimization (HPO) for FL is important because the choice of HPs can have dramatic impact on performance \citep{mcmahan2017communication}. 

While HPO has been widely studied in the centralized ML setting, it comes with unique challenges in the FL setting.  First, existing HPO techniques for centralized training often make use of the entire dataset, which is not available in FL. Secondly, they train a vast variety of models for a large number of HP configurations which would be prohibitively expensive in terms of communication and training time in FL settings. Thirdly, one important challenge that has not been adequately explored in FL literature is support for tabular data, which are widely used in enterprise settings \citep{ludwig2020ibm}.
One of the best models for this setting is based on gradient boosting tree algorithms~\citep{friedman2001greedy} which are different from the stochastic gradient descent algorithm used for neural networks.  Recently, a few approaches have been proposed for \flhpo,
however they focus on handling HPO using personalization techniques~\citep{khodak2021federated} and neural networks~\citep{khodak2020weight}. 
To the best of our knowledge, there is no HPO approach for FL systems to train non-neural network models, such as XGBoost~\citep{xgboost} that is particularly common in the enterprise setting.
%
\paragraph{Scope.} 
In this paper, we address the aforementioned challenges of \flhpo. 
We focus on the problem where the model HPs are shared across all parties
and we seek a set of HPs and train a single model that is eventually used by all parties for testing/deployment.
Moreover, we impose three further requirements that make the problem more challenging: 
%
(\cond{1})~we {\em do not make any assumption} that two models with different HPs can perform some form of ``weight-sharing'' (which is a common technique used in various HPO and neural architecture search (NAS) schemes for neural networks to reduce the computational overhead of HPO and NAS), allowing our solution to be applied beyond neural networks~\citep{khodak2020weight}. 
(\cond{2})~we seek to {\bf perform ``single-shot'' \flhpo}, where we have {\em limited} resources (in the form of computation and communication overhead) which allow training only a single model via federated learning (that is, a single HP configuration), and 
(\cond{3})~we {\em do not assume that parties have independent and identically distributed (IID) data distributions}.
%
%
\paragraph{Contributions.}
Given the above \flhpo problem setting, we make the following contributions:
\begin{itemize}[noitemsep,topsep=0pt]
\item (\S \ref{sec:method}) We present a novel framework {\em {\bf F}ederated {\bf Lo}ss Su{\bf R}face {\bf A}ggregation} (\flora) that leverages meta-learning techniques to utilize {\bf local and asynchronous} HPO on each party to perform single-shot HPO for the global \flhpo problem. 
\item (\S \ref{sec:og_analysis}) We provide theoretical guarantees for the set of HPs selected by \flora covering both IID and Non-IID cases. To the best of our knowledge, this is the first rigorous theoretical analysis for \flhpo problem and also the first optimality gap constructed in terms of the estimated loss given a target distribution.
\item (\S \ref{sec:emp}) We evaluate \flora on the \flhpo of Gradient Boosted Decision Trees~(GBDTs), Support Vector Machines~(SVMs) and Multi-layered Perceptrons~(MLPs) on seven classification datasets from OpenML~\citep{OpenML2013}, highlighting (i) its performance relative to the baseline, (ii) the effect of various choices in this scheme, and (iii) the effect of the number of parties on the performance.
\end{itemize}

\section{Related work}\label{sec:related}
%
%
\paragraph{Performance optimization of FL systems.} One of the main challenges in FL is achieving high accuracy and low communication overhead. FedAvg~\citep{pmlr-v54-mcmahan17a} is a predominant algorithm used for training in FL and several optimization schemes build on it. It is executed in multiple global rounds. At each round, the clients perform stochastic gradient descent (SGD) updates on their parameters based on their local objective functions. They subsequently send their updates to the server, which averages them and transmits their mean back to the clients. Several approaches have been devised for optimizing the communication performance of FL systems. Initially, communication optimizations included performing multiple SGD local iterations at the clients and randomly selecting a small subset of the clients to compute and send updates to the server~\citep{pmlr-v54-mcmahan17a}. Subsequently, compression techniques were used to minimize the size of model updates to the server. It has been shown that the accuracy and communication performance of these techniques depend highly on their HPs~\citep{pmlr-v54-mcmahan17a}.
\paragraph{\flhpo approaches.}
Recent optimization approaches adapt HPs such as the local learning rate at each client~\citep{Koskela19Learning,Mostafa19Robust,Reddi20Adaptive}, the number of local SGD iterations (which affect the frequency of server updates)~\citep{wang2019adaptive}. 
In~\citet{Dai20FBO,Dai21Differentially}, Dai et.al. address Federated Bayesian Optimization. Although using HPO with multiple HPs, the problem setup is quite different than Federated Learning: they focus on a single party using information from other parties to accelerate its own Bayesian Optimization, instead of building a model for all parties. Federated Network Architecture Search (FNAS) approaches search for architectural HPs of deep learning CNN models by running locally NAS algorithms and then aggregating the NAS architecture weights and model weights using FedAvg~\citep{He20toward,Garg20direct,Xu20federated}. These approaches have shown empirical gains but lack theoretical analysis. 
Inspired from the NAS technique of weight-sharing,~\citep{khodak2020weight, khodak2021federated} proposed FedEx, a \flhpo framework to accelerate a general HPO procedure, i.e., successive halving algorithm (SHA), for many SGD-based FL algorithms. Fedex focuses on building personalized models for parties by tuning local HPs of the parties. They provide a theoretical analysis for a special case of tuning a single HP (learning rate) in a convex optimization setting. 

Our framework improves on the above approaches in several ways. {\bf 1) It is more general}, as it can tune multiple HPs and is applicable to non SGD-training settings such as gradient boosting trees. This is achieved by treating \flhpo as a black-box HPO problem, which has been addressed in centralized HPO literature using grid search, random search~\citep{bergstra2012random} and Bayesian Optimization approaches~\citep{shahriari2016taking}.  The key challenge is the requirement to perform computationally intensive evaluations on a large number of HPO configurations, where each evaluation involves training a model and scoring it on a validation dataset. In the distributed FL setting  this problem is exacerbated because validation sets are local to the parties and each FL training/scoring evaluation is communication intensive.  Therefore a brute force application of centralized black-box HPO approaches that select hyper-parameters in an outer loop and proceed with FL training evaluations is not feasible.   
{2) \bf It yields minimal HPO communication overhead.} This is achieved by building a loss surface from local asynchronous HPO at the parties that yields a single optimized HP configuration used to train a global model with a single FL training. {\bf 3) It is the first that theoretically characterizes optimality gap in an \flhpo setting}, for the case we focus in this paper (creating a global model by tuning multiple global HPs).

\section{Methodology} \label{sec:method}

In the centralized ML setting, we would consider a model class $\mclass$ and its corresponding learning algorithm $\alg$ parameterized collectively with HPs $\btheta \in \bTheta$, and given a training set $D$, we can learn a single model $\alg(\mclass, \btheta, D) \to m \in \mclass$. Given some predictive loss $\Loss(m, D')$ of any model $m$ scored on some holdout set $D'$, the centralized HPO problem can be stated as
\begin{equation}
\label{eq:central-HPO}
\min\nolimits_{\btheta \in \bTheta} \Loss(\alg(\mclass, \btheta, D), D').
\end{equation}
In the most general FL setting, we have $p$ parties $P_1, \dots, P_p$ each with their private local training dataset $D_i, i \in [p] = \{1, 2, \ldots, p\}$. Let $D = \cup_{i=1}^p D_i$ denote the aggregated training dataset and $\overline{D} = \{D_i\}_{i \in [p]}$ denote the set of per-party datasets. Each model class (and corresponding learning algorithm) is parameterized by global HPs $\btheta_G \in \bTheta_G$ shared by all parties and per-party local HPs $\btheta_L^{(i)} \in \bTheta_L, i \in [p]$ with $\bTheta = \bTheta_G \times \bTheta_L$. FL systems usually include an aggregator with its own set of HPs $\bphi \in \bPhi$. Finally, we would have a FL algorithm 
\begin{equation}\label{eq:def-fl-alg}
    \falg \left( \mclass, \bphi, \btheta_G, \{ \btheta_L^{(i)} \}_{i \in [p]}, \alg, \overline{D} \right) \to m \in \mclass,
\end{equation}
which takes as input all the relevant HPs and per-party datasets and generates a model. in this case, the \flhpo problem can be stated in the two following ways depending on the desired goals: 
(i)~Ideally, for a global holdout dataset $D'$ (a.k.a validation set, possibly from the same distribution as the aggregated dataset $D$), the target problem is:
\begin{equation}
\label{eq:fl-hpo-1}
\min_{\bphi \in \bPhi, \btheta_G \in \bTheta_G, \btheta_L^{(i)} \in \bTheta_L, i \in [p]}
\Loss\left( \falg \left( 
\mclass, \bphi, \btheta_G, \{ \btheta_L^{(i)} \}_{i \in [p]}, \alg, \overline{D}
\right), D' \right).
\end{equation}
%
(ii)~An alternative target problem would involve per-party holdout datasets $D_i', i \in [p]$ as follows:
\begin{equation}
\label{eq:fl-hpo-2}
\min_{\bphi \in \bPhi, \btheta_G \in \bTheta_G, \btheta_L^{(i)} \in \bTheta_L, i \in [p]}
\mathsf{Agg}\left(
\left\{
  \Loss\left( \falg \left(
  \mclass, \bphi, \btheta_G, \{ \btheta_L^{(i)} \}_{i \in [p]}, \alg, \overline{D}
  \right), D'_i \right),
  i \in [p] \right\}
\right),
\end{equation}
%
where $\mathsf{Agg}: \R^p \to \R$ is some aggregation function (such as average or maximum) that scalarizes the $p$ per-party predictive losses.

%


Contrasting problem \eqref{eq:central-HPO} to problems \eqref{eq:fl-hpo-1} \& \eqref{eq:fl-hpo-2}, we can see that the \flhpo is significantly more complicated than the centralized HPO problem. In the ensuing presentation, we focus on problem~\eqref{eq:fl-hpo-1} although our proposed single-shot \flhpo scheme can be applied and evaluated for problem~\eqref{eq:fl-hpo-2}. We simplify the \flhpo problem in the following ways: (i)~we assume that there is no personalization so there are no per-party local HPs $\btheta_L^{(i)}, i \in [p]$, (ii)~we only focus on the model class HPs $\btheta_G$, deferring HPO for aggregator HPs $\bphi$ for future work, and (iii)~we assume there is a global holdout/validation set $D'$ which is only used to evaluate the final global model's performance but {\em can not be accessed} during HPO process.
Hence the problem we will study is stated as for a fixed aggregator HP $\bphi$:
\begin{equation} \label{eq:fl-hpo-1a}
\min\nolimits_{\btheta_G \in \bTheta_G}
\Loss\left( \falg \left( 
\mclass, \bphi, \btheta_G, \alg, \overline{D}
\right), D' \right).
\end{equation}
This problem appears similar to the centralized HPO problem~\eqref{eq:central-HPO}. However, note that the main challenges in \eqref{eq:fl-hpo-1a} is (i)~the need for a federated training for each set of HPs $\btheta_G$, and (ii)~the need to evaluate the trained model on the global validation set $D'$ (which is usually not available in usual \flhpo setting). Hence it is not practical (from a communication overhead and functional  perspective) to apply existing off-the-shelf HPO schemes to problem~\eqref{eq:fl-hpo-1a}. In the subsequent discussion, for simplicity purposes, we will use $\btheta$ to denote the global HPs, dropping the ``$G$'' subscript.

\subsection{Leveraging local HPOs}
\begin{algorithm}[t]
\caption{Single-shot \flhpo with Federated Loss Surface Aggregation (\flora)}
\label{alg:fl-hpo-lsa}
\begin{algorithmic}[1]
\STATE {\bf Input:}{$\bTheta, \mclass, \alg, \falg, \{ (D_i, D_i') \}_{i \in [p]}, T$}
  \FOR {each party $P_i, i \in [p]$}
    \STATE Run HPO to generate $T$ (HP, loss) pairs%
    \begin{equation}\label{eq:local-hpo-data}
    E^{(i)} = \left \{
    (\btheta_t^{(i)}, \Loss_t^{(i)}), t \in [T]
    \right\},
    \end{equation}%
    where $\btheta_t^{(i)} \in \bTheta,
    \Loss_t^{(i)} := \Loss(\alg(\mclass, \btheta_t^{(i)}, D_i), D_i')$.
  \ENDFOR
  \STATE Collect all $E = \{E^{(i)}, i \in [p] \}$ in aggregator 
  \STATE Generate a unified loss surface $\widehat \ell:\bTheta \to \R$ using $E$
  \STATE Select best HP candidate 
  \begin{align}\label{eq:def:flora-thetas}
      \widehat \btheta^\star \gets \arg \min\limits_{\btheta \in \bTheta} \widehat \ell(\btheta).
  \end{align} 
  \STATE Invoke federated training $m \gets \falg(\mclass, \widehat\btheta^\star, \alg, \overline{D})$
  \STATE {\bf Output:} FL model $m$.
\end{algorithmic}
\end{algorithm}
While it is impractical to apply off-the-shelf HPO solvers (such as Bayesian Optimization (BO)~\citep{shahriari2016taking}, Hyperopt~\citep{bergstra2011algorithms}, SMAC~\citep{hutter2011sequential}, and such), we wish to understand how we can leverage local and asynchronous HPOs in each of the parties. We begin with a simple but intuitive hypothesis underlying various meta-learning schemes for HPO~\citep{vanschoren2018meta,wistuba2018scalable}: {\em if a HP configuration $\btheta$ has good performance for all parties independently, then $\btheta$ is a strong candidate for federated training}.

With this hypothesis, we present our proposed algorithm {\bf \flora} in Algorithm~\ref{alg:fl-hpo-lsa}. In this scheme, we allow each party to perform HPO locally and asynchronously with some adaptive HPO scheme such as BO (line 3). Then, at each party $i \in [p]$, we collect all the attempted $T$ HPs $\btheta_t^{(i)}, t \in [T] = \{1, 2, \ldots, T\}$ and their corresponding predictive loss $\Loss_t^{(i)}$ into a set $E^{(i)}$ (line 3, equation~\eqref{eq:local-hpo-data}). Then these per-party sets of (HP, loss) pairs $E^{(i)}$ are collected at the aggregator (line 5). This operation has at most $O(pT)$ communication overhead (note that the number of HPs are usually much smaller than the number of columns or number of rows in the per-party datasets). These sets are then used to generate an aggregated loss surface $\widehat \ell: \bTheta \to \R$ (line 6) which will then be used to make the final single-shot HP recommendation $\widehat \btheta^\star \in \bTheta$ (line 7) for the federated training to create the final model $m \in \mclass$ (line 8). We will discuss the generation of the aggregated loss surface in detail in \S \ref{sec:method:lsa}. Before that, we briefly want to discuss the motivation behind some of our choices in Algorithm~\ref{alg:fl-hpo-lsa}.
\paragraph{Why adaptive HPO?}
The reason to use adaptive HPO schemes instead of non-adaptive schemes such as random search or grid search is that this allows us to efficiently approximate the local loss surface more accurately (and with more certainty) in regions of the HP space where the local performance is favorable instead of trying to approximate the loss surface well over the complete HP space. This has advantages both in terms of computational efficiency and loss surface approximation.
%
\paragraph{Why asynchronous HPO?}
Each party executes HPO asynchronously, without coordination with HPO results from other parties or with the aggregator. This is in line with our objective to minimize communication overhead. Although there could be strategies that involve coordination between parties, they could involve many rounds of communication. Our experimental results show that this approach is effective for the datasets we evaluated for.

%
\subsection{Loss surface aggregation} \label{sec:method:lsa}
Given the sets of (HP, loss) pairs $E^{(i)} = (\btheta_t^{(i)}, \Loss_t^{(i)}), i \in [p], t \in [T]$ at the aggregator, we wish to construct a loss surface $\widehat \ell : \bTheta \to \R$ that best emulates the (relative) performance loss $\widehat \ell(\btheta)$ we would observe when training the model on $\overline{D}$. Based on our hypothesis, we want the loss surface to be such that it would have a relatively low $\widehat \ell(\btheta)$ if $\btheta$ has a low loss for all parties simultaneously. However, because of the asynchronous and adaptive nature of the local HPOs, for any HP $\btheta \in \bTheta$, we would not have the corresponding losses from all the parties. For that reason, we will model the loss surfaces using regressors that try to map any HP to their corresponding loss. In the following, we present four ways of constructing such loss surfaces:
\paragraph{Single global model (SGM).}
We merge all the sets $ E= \cup_{i\in [p]} E^{(i)}$ and use it as a training set for a regressor $f: \bTheta \to \R$, which considers the HPs $\btheta \in \bTheta$ as the covariates and the corresponding loss as the dependent variable. For example, we can train a random forest regressor~\citep{breiman2001random} on this training set $E$. Then we can define the loss surface $\widehat \ell(\btheta) := f(\btheta)$. While this loss surface is simple to obtain, it may not be able to handle Non-iid party data distribution well: it is actually overly optimistic -- under the assumption that every party generates unique HPs during the local HPO, this single global loss surface would assign a low loss to any HP $\btheta$ which has a low loss at any one of the parties. This implies that this loss surface would end up recommending HPs that have low loss in just one of the parties, but not necessarily on all parties.
\paragraph{Single global model with uncertainty (SGM+U).}
Given the merged set $E= \cup_{i\in [p]} E^{(i)}$, we can train a regressor that provides uncertainty quantification around its predictions (such as Gaussian Process Regressor~\citep{williams2006gaussian}) as $f: \bTheta \to \R, u: \bTheta \to \R_+$, where $f(\btheta)$ is the mean prediction of the model at $\btheta \in \bTheta$ while $u(\btheta)$ quantifies the uncertainty around this prediction $f(\btheta)$. We define the loss surface as $\widehat \ell(\btheta) := f(\btheta) + \alpha \cdot u(\btheta)$ for some $\alpha > 0$. This loss surface does prefer HPs that have a low loss even in just one of the parties, but it penalizes a HP if the model estimates high uncertainty around this HP. Usually, a high uncertainty around a HP would be either because the training set $E$ does not have many samples around this HP (implying that many parties did not view the region containing this HP as one with low loss), or because there are multiple samples in the region around this HP but parties do not collectively agree that this is a promising region for HPs. Hence this makes SGM+U more desirable than SGM, giving us a loss surface that estimates low loss for HPs that are simultaneously thought to be promising to multiple parties.
\paragraph{Maximum of per-party local models (MPLM).}
Instead of a single global model on the merged set $E$, we can instead train a regressor $f^{(i)}: \bTheta \to \R, i \in [p]$ with each of the per-party set $E^{(i)}$. Given this, we can construct the loss surface as $\widehat \ell(\btheta) := \max_{i \in [p]} f^{(i)}(\btheta)$. This can be seen as a much more pessimistic loss surface, assigning a low loss to a HP only if it has a low loss estimate across all parties.
\paragraph{Average of per-party local models (APLM).}
A less pessimistic version of MPLM would be to construct the loss surface as the average of the per-party regressors $f^{(i)}, i \in [p]$ instead of the maximum, defined as $\widehat \ell(\btheta) := \nicefrac{1}{p} \sum_{i=1}^p f^{(i)}(\btheta)$. This is also less optimistic than SGM since it will assign a low loss for a HP only if its average across all per-party regressors is low, which implies that all parties observed a relatively low loss around this HP.

Intuitively, we believe that loss surfaces such as SGM+U or APLM would be the most promising while the extremely optimistic and pessimistic  SGM and MPLM respectively would be relatively less promising, with MPLM being superior to SGM. In the following section, we theoretically quantify the performance guarantees for MPLM and APLM, and in \S \ref{sec:emp},  we evaluate all these loss surface empirically in the single-shot \flhpo setting.

\section{Optimality analysis}\label{sec:og_analysis}
%
%
%
In this section, we provide a rigorous analysis of the sub-optimality of the HP selected by \flora.
Let us first define some notation we will use throughout this section.
\begin{definition}[Loss functions]
For a given set of parties' data
$\overline{D} = \{D_i\}_{i \in [p]}$ and any $\btheta \in \bTheta$, the true target loss (any predictive performance metric, such as, the training loss) can be expressed as:
\begin{equation}\label{eq:def:loss}
\loss(\btheta, \mathcal{D}) :=
 \underbrace{\E_{(x, y) \sim \mathcal{D}}}_{\text{test perf. of trained model}}
\Loss(\underbrace{\alg(\btheta, \overline D)}_{\text{trained model}}, (x, y)).
\end{equation}
Here $\mathcal{D}$ is the data distribution of the test points. Let $\tilde \loss(\btheta, \mathcal{D})$ be an estimate of the loss defined in \eqref{eq:def:loss} given some validation (holdout) set $D'$ sampled from $\mathcal{D}$, which is the model performance metric during evaluation and/or inference time.
\end{definition}
We assume the parties' training sets are collected before the federated learning such that $\overline D$ is fixed and unchanged during the HPO and FL processes, in order words, we do not consider streaming data setting.

Now we are ready to provide a more general definition of the unified loss surface constructed by \flora as follows:
\begin{definition}[Unified loss surface]
Given the local loss surfaces $\widehat \loss_i: \bTheta \to \R$ for each party $i \in [p]$ generated by $T$ (HP, loss) pairs 
$\{(\btheta^{(i)}_t, \Loss_t^{(i)}) \}_{t\in[T]}$, we can define the global loss surface $\widehat \loss : \bTheta \to \R$ as
\begin{equation}
    \label{eq:def:global-loss}
    \widehat \loss(\btheta) = \tsum_{i=1}^p \alpha_i(\btheta) \cdot \widehat \loss_i(\btheta), \alpha_i(\btheta) \in [0, 1], \tsum_{i=1}^{p} \alpha_i(\btheta) = 1.
\end{equation}
In particular, 
\begin{itemize}[noitemsep,topsep=0pt]
    \item[i)] If $\alpha_i(\btheta) = \nicefrac{1}{p}, \ \forall i \in [p], \btheta \in \bTheta$, then this reduces to APLM loss surface.
    \item[ii)] If $\alpha_i(\btheta) = \mathbb{I} \left (\widehat\loss_i(\btheta) = \max_{j \in [p]} \widehat\loss_j(\btheta) \right)$, then this reduces to the MPLM loss surface (assuming all $\widehat\ell_j(\btheta)$s are unique).
\end{itemize}
\end{definition}


We formalize the distance metric used in our analysis to evaluate the distance between two given data distributions.
\begin{definition}[1-Wasserstein distance~\citep{villani2003topics}]
For two distributions $\mu, \nu$ with bounded support, the 1-Wasserstein distance is defined as
\begin{equation} \label{eq:def:1wd}
    \mathcal W_1(\mu, \nu) := \sup_{f \in \mathsf F_1}
    \E_{x \sim \mu} f(x) - \E_{x \sim \nu} f(x),
\end{equation}
where $\mathsf F_1 = \{ f: f \text{ is continuous}, \textsf{Lipschitz}(f) \leq 1 \}$.
\end{definition}
%

To facilitate our analysis later, we make the following Lipschitzness assumptions regarding the loss function $\tilde \loss$ and also the per-party loss surface $\widehat \loss_i$.
\begin{assumption}[Lipschitzness]
For a fixed data distribution $\mathcal D$ and $\forall \btheta, \btheta' \in \bar{\bTheta} \subset \bTheta$, we have
\begin{align}
    | \tilde \loss(\btheta, \mathcal D) - \tilde \loss(\btheta', \mathcal D) |  &\leq \tilde L(\mathcal D) \cdot d(\btheta, \btheta'), \label{eq:def:theta-lip-tilde}\\
    | \widehat \loss_i(\btheta) - \widehat \loss_i(\btheta') |  &\leq \widehat L_i \cdot d(\btheta, \btheta'), \label{eq:def:theta-lip-hat}
\end{align}
where $d(\cdot, \cdot)$ is a certain distance metric defined over the hyper-parameter search space, see Appendix~\ref{app:def-dis-theta-space} for one definition.
For a fixed set of hyper-parameters $\btheta  \in \bar{\bTheta} \subset \bTheta$ and some data distributions $\mathcal D$ and $\mathcal D'$, we have
\begin{equation}\label{eq:def:dist-lip-hat}
    | \tilde \loss(\btheta, \mathcal D) - \tilde \loss(\btheta, \mathcal D') |  \leq \tilde \beta(\btheta) \cdot \mathcal W_1(\mathcal D, \mathcal D').
\end{equation}
\end{assumption}
%
\paragraph{Remark.}
Note that we explicitly use a $\bar\bTheta \subset \bTheta$ to highlight that we need Lipschitz-ness only in particular parts of the HP space. In fact, our analysis only requires the Lipschitz-ness at $\widehat \btheta^*$, the optimal HP and a HP space containing these two HPs and the set of HP tried in local HPO runs, i.e., $\btheta_t^{(i)}$, which most of the time does not cover the entire HP search space.
Moreover, the above Lipschitzness assumption w.r.t. a general HP space, which could be a combination of continuous and discrete variables, may be strong. We also show in Appendix~\ref{app:ct-theta-space} and \ref{app:alt-proof-prop-bnd-eps-i} that it can be relaxed to a milder assumption based on the modulus of continuity without significantly affecting our main results.
For simplicity, we can always assume that
$\tilde L(\mathcal D)  \leq \tilde L, \ \forall \mathcal D$ and 
$\tilde \beta(\btheta) \leq \tilde \beta, \ \forall \btheta$.

Recall that the HP $\widehat\btheta^\star $ selected by \flora is defined as in \eqref{eq:def:flora-thetas}.
We then define the optimal HP given by the estimated loss function for a desired data distribution $\mathcal{D}$ we want to learn as,
\begin{equation}\label{eq:def:global-theta}
\btheta^\star \in \arg \min_{\btheta \in \bTheta} \tilde \loss(\btheta, \mathcal{D}).
\end{equation}


We are interested in providing a bound for the following {\em optimality gap}:
\begin{equation} \label{eq:def:og}
    \tilde \loss(\widehat \btheta^\star, \mathcal{D}) - \tilde \loss(\btheta^\star, \mathcal{D}).
\end{equation}
Note that this bound is the optimality gap for the output of \flora in terms of the estimated loss $\tilde \loss$. We state our main results in the following theorem.
Informally speaking we show how to bound the optimality gap by picking the `worst-case' HP setting that maximizes the combination of Wasserstein distances of the local data distributions and actual quality of local HPO approximation across parties.

\begin{theorem}\label{thm:og-bnd}
Consider the optimality gap defined in \eqref{eq:def:og}, where $\widehat \btheta^*$ is selected by \flora with each party $i\in [p]$ collecting $T$ (HP, loss) pairs $\{(\btheta_t^{(i)}, \Loss_t^{(i)})\}_{t\in [T]}$  during the local HPO run. 
For a desired data distribution $ \mathcal{D} = \sum_{i=1}^p w_i \mathcal{D}_i$, where $\{\mathcal{D}_i\}_{i\in [p]}$ are the sets of parties' local data distributions and $w_i\in [0, 1], \forall i\in [p]$, we have 
\begin{align}\label{eq:og-bnd}
     &\tilde \loss(\widehat \btheta^\star, \mathcal{D}) - \tilde \loss(\btheta^\star, \mathcal{D})\nonumber\\
     & \ \le 2 \max_{\btheta \in \bar\bTheta}\tsum_{i\in [p]} \alpha_i(\btheta)\left\{ 
     \tilde \beta(\btheta) 
     \tsum_{j \in [p], j \not= i} w_j \mathcal W_1 (\mathcal{D}_j, \mathcal{D}_i)
     + \left( \tilde L(\mathcal{D}_i) + \widehat L_i \right)
   \min_{t \in [T]} d(\btheta, \btheta_t^{(i)})
  + \delta_i\right\},
\end{align}
where $\delta_i$ is the maximum per sample training error for the local loss surface $\widehat \loss_i$, i.e., $\delta_i = \max_t | \Loss_t^{(i)} - \widehat \loss_i (\btheta_t^{(i)}) |$. 
In particular, when all parties have i.i.d. local data distributions, \eqref{eq:og-bnd} reduces to
\begin{align*}
     & \tilde \loss(\widehat \btheta^\star, \mathcal{D}) - \tilde \loss(\btheta^\star, \mathcal{D})
     \le 2 \max_{\btheta \in \bar\bTheta}\sum_{i=1}^{p} \alpha_i(\btheta) \left\{\left( \tilde L(\mathcal{D}_i) + \widehat L_i \right) \min\limits_{t \in [T]} d(\btheta, \btheta_t^{(i)})
  + \delta_i\right\}.
\end{align*}
\end{theorem}
We make some observations regarding the above results. Firstly, the first term in our bound characterizes the errors incurred by the differences among parties' local data distributions, i.e., the magnitude of Non-IIDness in a FL system. In particular, we can see it vanish under the IID setting.
Secondly, the last two terms measure the quality of the local HPO approximation, which can be reduced if a good loss surface is selected.
Thirdly, $\min_{t \in [T]} d(\btheta, \btheta_t^{(i)})$
indicates that the optimality gap depends only on the HP trials $\btheta_t^{(i)}$ that is closest to the optimal HP setting.
Finally, if we assume each party's training dataset $D_i$ is of size $n_i$ sampled as $D_i \sim \mathcal{D}_i^{n_i}$, we can view $w_i = \tfrac{n_i}{n}$ where $n = \sum_{i=1}^p n_i$, i.e., with probability $w_i$ the desired data distribution $\mathcal{D}$ is sampled from $\mathcal{D}_i$.


In order to obtain the result in \eqref{eq:og-bnd}, we first analyze \eqref{eq:def:og} in Proposition~\ref{prop:opt-bnd}, see its proof in Appendix~\ref{app:proofs}.
Note that the local loss surfaces $\widehat \loss_i, \ i\in[p]$ are computed at a certain test/validation set sampled from the parties local data distribution $\mathcal{D}_i$. We quantify the relationship between $\widehat \loss_i(\btheta)$ and the estimated loss function $\tilde \loss(\btheta, \mathcal{D}_i)$ as follows:
\begin{equation}
    \label{eq:def:local-loss}
    | \widehat \loss_i(\btheta) - \tilde \loss(\btheta, \mathcal{D}_i) | := \epsilon_i(\btheta, T).
\end{equation}

\begin{proposition}\label{prop:opt-bnd}
Consider $\widehat \btheta^\star$ and $\btheta^\star$ are two sets of HP defined in \eqref{eq:def:flora-thetas} and \eqref{eq:def:global-theta}, respectively, and $\{\mathcal{D}_i\}_{i\in [p]}$ and $\mathcal{D}$ are the sets of parties' local data distributions and the target (global) data distribution we want to learn, for a given HP space such that $\widehat \btheta^\star, \btheta^\star \in \bar\bTheta \subset \bTheta$, we have
\begin{align}
    &\tilde \loss(\widehat \btheta^\star, \mathcal{D}) - \tilde \loss(\btheta^\star, \mathcal{D}) 
    \label{eq:def:opt-bnd}
    \le 2 \max_{\btheta \in \bar\bTheta}\tsum_{i\in [p]} \alpha_i(\btheta) \left\{\tilde \beta(\btheta) \mathcal W_1(\mathcal{D}, \mathcal{D}_i) +  \epsilon_i(\btheta, T)\right\}.
\end{align}
\end{proposition}

We now dive into each term in \eqref{eq:def:opt-bnd} to provide tight bounds for $\mathcal W_1(\mathcal{D}, \mathcal{D}_i)$ and $\epsilon_i(\btheta, T)$ in the following propositions. All the proofs can be found in Appendix~\ref{app:proofs}.

\begin{proposition}\label{prop:bnd-wd}
Consider 1-Wasserstein distance we defined in \eqref{eq:def:1wd}, for a local data distribution $\mathcal{D}_i$ of any party $i, \ i\in [p]$, and $ \mathcal{D} = \sum_{i=1}^p w_i \mathcal{D}_i$ for some $w_i\in [0,1], \forall i\in [p]$, we have
\begin{equation}\label{eq:bnd-wd}
    \mathcal W_1(\mathcal{D}, \mathcal{D}_i) \leq \tsum_{j \in [p], j \not= i} w_j \mathcal W_1 (\mathcal{D}_j, \mathcal{D}_i). 
\end{equation}
In particular, when $\mathcal{D}_i, \ i\in [p]$ are i.i.d. data distribution, i.e., all parties in a federated learning system possess i.i.d. local data distribution -- that is, $\mathcal W_1(\mathcal{D}_j, \mathcal{D}_i) = 0 \forall i,j \in [p]$ -- then  $\tsum_{j\in [p], j \not= i} w_j \mathcal W_1 (\mathcal{D}_j, \mathcal{D}_i) = 0$. Therefore, $\mathcal W_1(\mathcal{D}, \mathcal{D}_i) = 0, \ \forall i\in [p]$.
\end{proposition}
\begin{proposition}\label{prop:bnd-eps-i}
For any party $i, \ i\in[p]$, consider a (HP, loss) pair $(\btheta_t^{(i)}, \Loss_t^{(i)})$ collected during the local HPO run for party $i$, for any $\btheta \in \bar{\bTheta} \subset \bTheta$, we have
\begin{equation}\label{eq:bnd-eps-i}
\epsilon_i(\btheta, T)\leq 
   \left( \tilde L(\mathcal{D}_i) + \widehat L_i \right)
   \min_{t \in [T]} d(\btheta, \btheta_t^{(i)})
  + \delta_i,
\end{equation}
where $\delta_i = \max_t | \Loss_t^{(i)} - \widehat \loss_i (\btheta_t^{(i)}) |$
is the maximum per sample training error for the local loss surface $\widehat \loss_i$.
\end{proposition}

Note that if we use non-parametric regression models as the loss surfaces (such as Gaussian Processes, Random Forests, etc), the per-sample training error can be made arbitrarily small (that is $\delta_i \approx 0$), but at the cost of increasing $\widehat L_i$ for $\widehat \loss_i$.

\section{Empirical evaluation} \label{sec:emp}

\begin{table*}[t]
\centering
\caption{Comparison of different loss surfaces (the 4 rightmost columns)  for \flora relative to the  baseline for single-shot 3-party \flhpo in terms of the {\em relative regret} (lower is better). 
}
{\scriptsize
\begin{tabular}{lcccccc}
\toprule
Aggregate & ML Method & SGM & SGM+U & MPLM & APLM \\
\midrule
Regret &  HGB  &   [0.30, 0.47, 0.68]  &   [0.27, 0.54, 0.64]  &   [0.25, 0.43, 0.67]  &   [0.25, 0.50, 0.65] \\
Inter-quartile range &  SVM  &   [0.04, 0.38, 1.11]  &   [0.04, 0.48, 1.07]  &   [0.38, 0.91, 2.41]  &   [0.23, 0.54, 0.76] \\
&  MLP  &   [0.36, 0.80, 0.97]  &   [0.48, 0.99, 1.01]  &   [0.47, 0.89, 1.00]  &   [0.46, 0.79, 0.95] \\
& Overall & [{\bf 0.22}, {\bf 0.53}, 0.97] & [0.32, 0.55, 1.01] & [0.36, 0.61, 0.99] & [0.36, 0.57, {\bf 0.79}] \\
\midrule
\flora & HGB &  6/0/1  &   6/0/1  &   7/0/0  &   7/0/0 \\
Wins/Ties/Losses & SVM &  4/0/2  &   4/0/2  &   3/0/3  &   5/0/1 \\
& MLP &   6/0/1  &   4/1/2  &   5/1/1  &   6/0/1 \\
& Overall & 16/0/4 & 14/1/5 & 15/1/4 & {\bf 18/0/2} \\
\midrule
Wilcoxon Signed-Rank Test &  HGB & (26, 0.02126)  &   (27, 0.01400)  &   (28, 0.00898)  &   (28, 0.00898) \\
1-sided & SVM  &   (18, 0.05793)  &   (17, 0.08648)  &   (9, 0.62342)  &   (15, 0.17272) \\
(statistic, p-value) & MLP  &  (21, 0.11836)  &   (15, 0.17272)  &   (18, 0.05793)  &   (24, 0.04548) \\
& Overall & (174, 0.00499) & (164, 0.00272) & (141, 0.03206) & {\bf (183.5, 0.00169)}
 \\
\bottomrule
\end{tabular}
%
}
\label{tab:baseline-comp}
\end{table*}

In this section, we evaluate our proposed scheme \flora with different loss surfaces for the \flhpo on a variety of ML models -- histograms based gradient boosted (HGB) decision trees~\citep{friedman2001greedy}, Support Vector Machines (SVM) with RBF kernel and multi-layered perceptrons (MLP) (using their respective {\tt scikit-learn} implementation~\citep{scikit-learn}) on OpenML~\citep{OpenML2013} classification problems. 
The precise HP search space is described in Appendix~\ref{asec:emp:search-space}. 
First, we fix the number of parties $p = 3$ and compare \flora to a baseline on $7$ datasets. Then we study the effect of increasing the number of parties from $p = 3$ up to $p = 100$ on the performance of our proposed scheme on $3$ datasets. The data is randomly split across parties.
We also evaluate \flora with different parameter choices, in particular, the number of local HPO rounds and the communication overhead in the aggregation of the per-party (HP, loss) pairs.
Finally, we evaluate \flora in a real FL testbed IBM FL~\citep{ludwig2020ibm} using its default HP setting as a baseline.
More experimental results can be found in Appendix~\ref{app:exp}.
%
\paragraph{Single-shot baseline.}
To appropriately evaluate our proposed single-shot \flhpo scheme, we need to select a meaningful single-shot baseline. For this, we choose the default HP configuration of {\tt scikit-learn} 
as the single-shot baseline for two main reasons: (i)~the default HP configuration in {\tt scikit-learn} is set manually based on expert prior knowledge and extensive empirical evaluation, and 
(ii)~these are also used as the defaults in the 
Auto-Sklearn package~\citep{feurer2015efficient,feurer-arxiv20a}, one of the leading open-source AutoML python packages, which maintains a carefully selected portfolio of default configurations.
%
\paragraph{Dataset selection.}
For our evaluation of single-shot HPO, we consider 
$7$ binary classification datasets of varying sizes and characteristics from OpenML~\citep{OpenML2013} such that there is at least a significant room for improvement over the single-shot baseline performance. We consider datasets which have at least $> 3\%$ potential improvement in balanced accuracy for gradient boosted decision trees. See Appendix~\ref{asec:emp:datasets} for details on data.
Note that this only ensures room for improvement for HGB, while highlighting cases with no room for improvement for SVM and MLP as we see in our results.
\paragraph{(Dis-)Regarding other baselines.}
While there are some existing schemes for \flhpo (as discussed in \S \ref{sec:related}), we are unable to compare \flora to them for the following reasons: (i)~As noted by \citet[Section 1, Related Work]{khodak2021federated}, existing schemes focus ``on a small number of
hyperparameters (e.g. the step-size and sometimes one or two more) in less general settings (studying small-scale problems or assuming server-side validation data)'' whereas we explicitly assume no access to such a ``server-side validation data''. (ii)~Furthermore, we noted in \S \ref{sec:prob-setup} (\cond{1}), we do not assume any ``weight-sharing'' type capability, and hence it is not clear how FedEx~\citep{khodak2021federated} can be applied to \flhpo in the general\footnote{Moreover, \citep{khodak2021federated} claim that FedEx can handle architectural hyper-paramters but it is never demonstrated and discussed explicitly. In contrast, our proposed algorithm can handle architectural hyperparameters (as we do with HGB (tree depth) and MLP (width of the layer)).}. 
\paragraph{Implementation.}
We consider two implementations for our empirical evaluation. In our first three sets of experiments, we emulate the final FL (Algorithm~\ref{alg:fl-hpo-lsa}, line 8) with a centralized training using the pooled data. We chose this implementation because we want to evaluate the final performance of any HP configuration (baseline or recommended by \flora) in a statistically robust manner with multiple train/validation splits (for example, via 10-fold cross-validation) instead of evaluating the performance on a single train/validation. This form of evaluation is extremely expensive to perform in a real FL system and generally not feasible, but 
allows us to evaluate how the performance of our single-shot HP recommendation fairs against that of the best-possible HP found via a full-scale centralized HPO.
\paragraph{Evaluation metric.}
In all datasets, we consider the balanced accuracy as the metric we wish to maximize. For the local per-party HPOs (as well as the centralized HPO we execute to compute the regret), we maximize the 10-fold cross-validated balanced accuracy. For Table~\ref{tab:baseline-comp}-\ref{tab:num-parties-comp}, we report the {\em relative regret}, computed as $\nicefrac{(a^\star - a)}{(a^\star - b)}$, where $a^\star$ is the best metric obtained via the centralized HPO, $b$ is the result of the baseline, and $a$ is the result of the HP recommended by \flora. The baseline has a relative regret of 1 and smaller values imply better performance. A value larger than 1 implies that the recommended HP performs worse than the baseline. 
%
\paragraph{Comparison to single-shot baseline.}
In our first set of experiments for 3-party \flhpo ($p = 3$), we compare our proposed scheme with the baseline across different datasets, machine learning models and \flora loss surfaces. The aggregated results are presented in Table~\ref{tab:baseline-comp}, with the individual results detailed in Appendix~\ref{asec:baseline-comp}. For each of the three methods, we report the aggregate performance over all considered datasets in terms of (i)~inter-quartile range, (ii)~Wins/Ties/Losses of \flora w.r.t. the single-shot baseline, and (iii)~a one-sided Wilcoxon Signed Ranked Test of statistical significance with the null hypothesis that the median of the difference between the single-shot baseline and \flora is positive against the alternative that the difference is negative (implying \flora improves over the baseline). Finally, we also report an ``Overall'' performance, further aggregated across all ML models.

All \flora loss surfaces show strong performance w.r.t the single-shot baseline, with significantly more wins than losses, and 3rd-quartile relative regret values less than 1 (indicating improvement over the baseline). All \flora loss surfaces have a p-value of less than $0.05$, indicating that we can reject the null hypothesis. 
Overall, APLM shows the best performance over all loss surfaces, both in terms of Wins/Ties/Losses over the baseline as well as in terms of the Wilcoxon Signed Rank Test, with the highest statistic and a p-value close to $10^{-3}$. APLM also has significantly lower 3rd-quartile than all other loss surfaces. MPLM appears to have the worst performance but much of that is attributable to a couple of very hard cases with SVM (see Appendix~\ref{asec:baseline-comp} for detailed discussion). Otherwise, MPLM performs second best both for \flhpo with HGB and MLP.
\begin{table}[t]
\centering
\caption{Effect of increasing the number of parties on \flora with different loss surfaces for HGB. 
}
{\footnotesize
\begin{tabular}{lcccccc}
\toprule
Data                     & $p$      &$\gamma_p$& SGM & SGM+U & MPLM & APLM \\
\midrule
EEG Eye State            & 3        &  1.01 & 0.14   & 0.12   & 0.11   & 0.12 \\
14980 samples            & 6        &  1.01 & 0.07   & 0.00   & 0.07   & 0.09 \\
                         & 10       &  1.03 & 0.08   & 0.00   & 0.16   & 0.01 \\
                         & 25       &  1.08 & 0.35   & 0.92   & 0.17   & 0.04 \\
                         & 50       &  1.20 & 0.20   & 0.23   & 0.67   & 0.12 \\

\midrule
Electricity               & 3       &  1.01 & 0.17   & 0.14   & 0.09   & 0.12 \\
45312 samples             & 6       &  1.01 & 0.25   & 0.21   & 0.18   & 0.13 \\
                          & 10      &  1.02 & 0.03   & 0.06   & 0.32   & 0.14 \\
                          & 25      &  1.04 & 0.40   & 0.42   & 1.42   & 0.89 \\
                          & 50      &  1.07 & 1.57   & 1.57   & 0.89   & 1.13 \\
                          & 100     &  1.14 & 1.45   & 1.47   & 0.48   & 1.11 \\
\midrule
Pollen                    & 3       &  1.02 & 0.43   & 0.54   &  0.43  & 0.69 \\
3848 samples              & 6       &  1.10 & 1.02   & 0.91   &  0.54  & 0.56 \\
                          & 10      &  1.16 & 1.05   & 0.73   &  0.75  & 1.12 \\
\bottomrule
\end{tabular}
}
\label{tab:num-parties-comp}
\end{table}
%
\paragraph{Effect of increasing number of parties.}
In the second set of experiments, we study the effect of increasing the number of parties in the \flhpo problem on 3 datasets and HGB. For each data set, we increase the number of parties $p$ up until each party has at least 100 training samples. We present the relative regrets in Table~\ref{tab:num-parties-comp}. It also displays
$\gamma_p := \nicefrac{\left(1 - \min_{i \in [p]} \Loss_\star^{(i)} \right) }{ \left(1 - \max_{i \in [p]} \Loss_\star^{(i)} \right)}$, where $\Loss_\star^{(i)} = \min_{t \in [T]} \Loss_t^{(i)}$ is the minimum loss observed during the local asynchronous HPO at party $i$. This ratio $\gamma_p$ is always greater than 1, and highlights the difference in the observed performances across the parties. A ratio closer to 1 indicates that all the parties have relatively similar performances on their respective training data, while a ratio much higher than 1 indicating significant discrepancy between the per-party performances, implicitly indicating the difference in the per-party data distributions. 

We notice that increasing the number of parties does not have a significant effect on $\gamma_p$ for the Electricity dataset until $p=100$, but significantly increases for the Pollen dataset earlier (making the problem harder). For the EEG eye state, the increase in $\gamma_p$ with increasing $p$ is moderate until $p=50$. 
The results indicate that, with low or moderate increase in $\gamma_p$ (EEG eye state, Electricity for moderate $p$), the proposed scheme is able to achieve low relative regret -- the increase in the number of parties does not directly imply degradation in performance. However, with significant increase in $\gamma_p$ (Pollen, Electricity with $p=50, 100$ and EEG Eye State with $p=50$), we see a significant increase in the relative regret (eventually going over 1 in a few cases). 
In this challenging case, MPLM (the most pessimistic loss function) has the most graceful degradation in relative regret compared to the remaining loss surfaces.
%
%
\begin{figure}
\centering
\begin{subfigure}{0.45\textwidth}
\includegraphics[width=\textwidth]{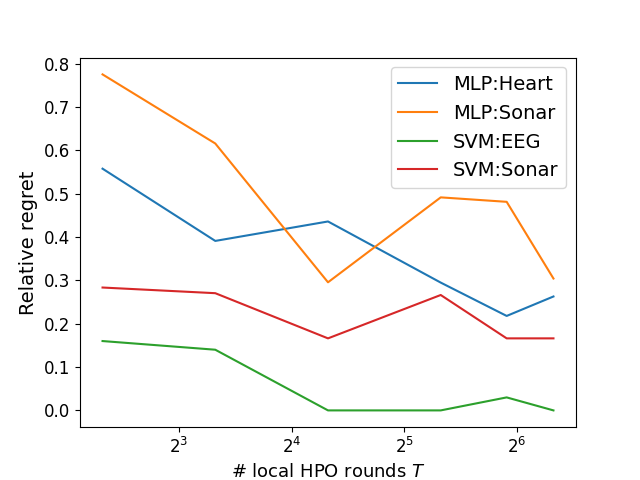}
\caption{\# local HPO rounds.}
\label{fig:aplm-T-dep}
\end{subfigure}
~
\begin{subfigure}{0.45\textwidth}
\includegraphics[width=\textwidth]{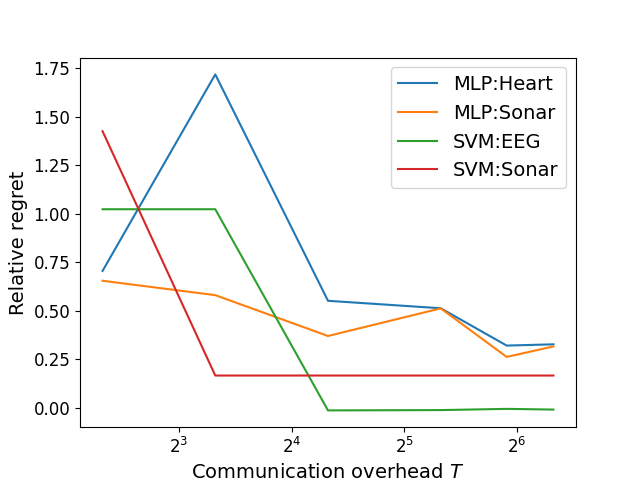}
\caption{\# (HP, loss) pairs communicated to aggregator}
\label{fig:aplm-comm}
\end{subfigure}
\caption{Effect of different choices on \flora with the APLM loss surface for different methods and datasets. More results and other loss surfaces are presented in Appendix~\ref{asec:emp:Tdep} and \ref{asec:emp:Tcommdep}.}
\label{fig:aplm-ablation}
\end{figure}
%
\paragraph{Effect of different choices in \flora.}
In this set of experiments, we consider \flora with the APLM loss surface, and ablate the effect of different choices in \flora on 2 datasets each for SVM and MLP. First, we study the impact of the thoroughness of the per-party local HPOs, quantified by the number of HPO rounds $T$ in Figure~\ref{fig:aplm-T-dep}. The results indicate that for really small $T$ ($<20$) the relative regret of \flora can be very high. However, after that point, the relative regret converges to its best possible value. We present the results for other loss surfaces in Appendix~\ref{asec:emp:Tdep}.

We also study the effect of the communication overhead of \flora for fixed level of local HPO thoroughness. We assume that each party performs $T=100$ rounds of local asynchronous HPO. However, instead of sending all $T$ (HP, loss) pairs, we consider sending $T' < T$ of the ``best'' (HP, loss) pairs -- that is, (HP, loss) pairs with the $T'$ lowest losses. Changing the value of $T'$ trades off the communication overhead of the \flora step where the aggregators collect the per-party loss pairs (Algorithm~\ref{alg:fl-hpo-lsa}, line 5). The results for this study are presented in Figure~\ref{fig:aplm-comm}, and indicate that, for really small $T'$, the relative regret can be really high. However, for a moderately high value of $T' < T$, \flora converges to its best possible performance. Results on other loss surfaces and further discussion can be found in Appendix~\ref{asec:emp:Tcommdep}.
\begin{table*}[t]
\centering
\caption{Performance of \flora with the IBM-FL system in terms of the {\em balanced accuracy} on a holdout test set (higher is better). The baseline is still the default HP configuration of {\tt HistGradientBoostingClassifier} in {\tt scikit-learn}.}
{\footnotesize
\begin{tabular}{lccccccc}
\toprule
Data          & \# parties & \# training data per party & Baseline & SGM & SGM+U & MPLM & APLM  \\
\midrule
Oil spill       &  $3$   &$200$   & 0.5895 & \textbf{0.7374} & 0.5909 & 0.7061 & 0.7332 \\
\midrule
EEG eye state &  $3$ & $3,000$  & 0.8864 & 0.9153 & 0.9211 & \textbf{0.9251} & 0.9245 \\
              
\midrule
Electricity   &  $6$  & $4,000$     & 0.8448 & 0.8562 & \textbf{0.8627} & 0.8621 & 0.8624 \\
             
\bottomrule
\end{tabular}
}
\label{tab:fl-summary}
\end{table*}
\paragraph{Federated Learning testbed evaluation.}
We now conduct experiments for histrogram boosted tree model in a FL testbed, utilizing IBM FL library~\citep{ludwig2020ibm,ong2020adaptive},
More specifically, we reserved $40\%$ of oil spill and electricity and $20\%$ of EEG eye state as global hold-out set only for evaluating the final FL model performance. Each party randomly sampled from the rest of the original dataset to obtain their own training dataset. 
We use the same HP search space as in Appendix~\ref{asec:emp:search-space}. 
We report the balanced accuracy of any HP (baseline or recommended by \flora) on a single train/test split. Given balanced accuracy as the evaluation metric, we utilize (1 - balanced accuracy) as the loss $\Loss_t^{(i)}$ in Algorithm~\ref{alg:fl-hpo-lsa}
Each party will run HPO to generate $T = 500$ (HP, loss) pairs and use those pairs to generate loss surface either collaboratively or by their own according to different aggregation procedures described in \S \ref{sec:method:lsa}.
Once the loss surface is generated, the aggregator uses Hyperopt~\citep{bergstra2011algorithms} to select the best HP candidate and train a federated XGBoost model via the IBM FL library using the selected HPs.
Table~\ref{tab:fl-summary}
summarizes the experimental results for $3$ datasets, indicating that \flora can significantly improve over the baseline in IBM FL testbed.

\section{Conclusions} \label{sec:next-steps}
How to effectively select hyper-parameters in FL settings is a challenging problem. In this paper, we introduced \flora, a single-shot \flhpo algorithm that can be applied to a variety of ML models. We provided a theoretical analysis which includes a bound on the optimality gap incurred by the hyper-parameter selection performed by \flora.  Our experimental evaluation shows that \flora can effectively produce hyper-paramater configurations that outperform the baseline with just a single shot.







\bibliography{reference}
\bibliographystyle{plainnat}

\newpage
\appendix
\onecolumn
\section{Technical Definitions}

\subsection{Distance in $\bTheta$}\label{app:def-dis-theta-space}

Here we will define a distance metric $d: \bTheta \times \bTheta \to \R_+$. Assuming we have $m$ HPs, if $\bTheta \subset \R^m$, then there are various distances available such as $\| \btheta - \btheta' \|_\rho$ (the $\rho$-norm). The more general case is where we have $R$ continuous/real HPs, $I$ integer HPs, and $C$ categorical HPs; $m = R + I + C$. In that case, $\bTheta \subset \R^{R} \times \Z^{I} \times \C^{C}$, and any $\btheta = (\btheta_{\R}, \btheta_{\Z}, \btheta_{\C}) \in \bTheta_{\R} \times \bTheta_{\Z} \times \bTheta_{\C} $, where $\btheta_{\R} \in \bTheta_{\R}, \btheta_{\Z} \in \bTheta_{\Z}, \btheta_{\C} \in \bTheta_{\C}$ respectively denote the continuous, integer and categorical HPs in $\btheta$. Distances over $\R^{R} \times \Z^{I}$ is available, such as $\rho$-norm. Let $d_{\R,\Z} : (\bTheta^{\R} \times \bTheta_{\Z}) \times (\bTheta_{\R} \times \bTheta_{\Z}) \to \R_+$ be some such distance.

To define distances over categorical spaces, there are some techniques such as one described by \citet{oh2019combinatorial}:

Assume that each of the $C$ HPs $\btheta_{\C,k}, k \in [C]$ have $n_k$ categories $\{\xi_{k1}, \xi_{k2}, \ldots, \xi_{kn_k} \} $. Then we define a complete undirected graph $G_k = (V_k, E_k), k \in [C]$ where 
\begin{itemize}
    \item There is a node $N_{kj}$ in $G_k$ for each category $\xi_{kj}$ for each $j \in [n_k]$ and $V_k = \{N_{k1}, \ldots N_{k n_k}\}$.
    \item There is an undirected edge $(N_{kj}, N_{kj'})$ for each pair $j, j' \in [n_k]$, and $E_k = \{ (N_{kj}, N_{kj'}), j, j' \in [n_k] \} $.
\end{itemize}

Given the per-categorical HP graph $G_k, k \in [C]$, we define the graph Cartesian product $\mathsf G = \bigotimes_{k \in [C]} G_k $ and $\mathsf G = (\mathsf V, \mathsf E)$ such that
\begin{itemize}
    \item $\mathsf V = \{ \mathsf N_{(j_1, j_2, \ldots, j_C)}: (\xi_{1 j_1}, \xi_{2 j_2}, \ldots \xi_{k j_k}, \ldots, \xi_{C j_{C}}) \in \bTheta_{\C}, j_k \in [n_k] \forall k \in [C] \}$.
    \item $\mathsf E = \{ (\mathsf N_{(j_1, j_2, \ldots, j_C)}, \mathsf N_{(j'_1, j'_2, \ldots, j'_C)}): \text{\sf IFF} \exists t \in [C] \text{ such that } \forall k \not= t, \xi_{k j_k} = \xi_{k j'_k}, \text{ and } \exists (N_{t j_t}, N_{t j'_t}) \in E_t \}$.
\end{itemize}

Then for any $\btheta_{\C}, \btheta'_{\C} \in \bTheta_{\C}$ with corresponding nodes $\mathsf N, \mathsf N' \in \mathsf V$, \citet[Theorem 2.2.1]{oh2019combinatorial} says that the length of the shortest path between nodes $\mathsf N$ and $\mathsf N'$ in $\mathsf G$ is a distance. We can consider this distance as $d_{\C}: \bTheta_{\C} \times \bTheta_{\C} \to \R_+$. Of course, there are other ways of defining distances in the categorical space.

Then we can define a distance $d : (\bTheta_{\R} \times \bTheta_{\Z} \times \bTheta_{\C}) \times  (\bTheta_{\R} \times \bTheta_{\Z} \times \bTheta_{\C}) \to \R_+$ between two HPs $\btheta, \btheta'$ as 
\begin{equation}\label{eq:mixed-dist}
d(\btheta, \btheta') = d_{\R, \Z}((\btheta_{\R}, \btheta_{\Z}), (\btheta'_{\R}, \btheta'_{\Z})) + d_{\C}(\btheta_{\C}, \btheta'_{\C}).
\end{equation}

\begin{proposition}
Given distance metrics $d_{\R,\Z}$ and $d_{\C}$, the function $d: \bTheta \times \bTheta $ defined in \eqref{eq:mixed-dist} is a valid distance metric.
\end{proposition}

\subsection{Continuity in the space of HPs $\bTheta$}\label{app:ct-theta-space}

In the simple case, we can assume Lipschitz continuity of estimated loss $\tilde \loss(\btheta, \mathcal D)$ and the loss surface $\widehat \loss_i(\btheta), i \in [p]$ as follows:
\begin{align}
| \loss(\btheta, \mathcal D) - \loss(\btheta', \mathcal D) | & \leq \tilde L(\mathcal D) \cdot d(\btheta, \btheta'),
\\
| \widehat \loss_i(\btheta) - \widehat \loss_i (\btheta') | & \leq \widehat L \cdot d(\btheta, \btheta').
\end{align}

For a more general handling, we can consider the notion of modulus of continuity in the form of a increasing real-valued functions $\omega : \R_+ \to \R_+$ with $\lim_{t \to 0} \omega(t) = \omega(0) = 0$. Then we can say that the estimated loss $\tilde \loss(\btheta, \mathcal D)$ and the loss surface $\loss_i (\btheta)$ admits $\tilde \omega_{\mathcal D}$ and $\widehat \omega$ as a modulus of continuity (respectively) if

\begin{align}
\label{eq:est-loss-moc}
| \loss(\btheta, \mathcal D) - \loss(\btheta', \mathcal D) | & \leq \tilde \omega_{\mathcal D}( d(\btheta, \btheta') )
\\
\label{eq:lsurf-moc}
| \widehat \loss_i(\btheta) - \widehat \loss_i (\btheta') | & \leq \widehat \omega( d(\btheta, \btheta') ).
\end{align}

If we further assume that $\tilde \omega_{\mathcal D}, \widehat \omega$ to be concave, then we can say that these functions are sublinear as follows:

\begin{align}
\tilde \omega_{\mathcal D} (t) & \leq  \tilde A_{\mathcal D} \cdot t + \tilde B_{\mathcal D},
\\
\widehat \omega (t) & \leq  \widehat A \cdot t + \widehat B.
\end{align}

These conditions give us (indirectly) something similar in spirit to the guarantees of Lipschitz continuity, but is a more rigorous way of achieving such guarantees.

\section{Proofs for optimality analysis}\label{app:proofs}
We provide detailed proofs of the propositions stated in Section~\ref{sec:og_analysis}.

\subsection{Proof of Proposition~\ref{prop:opt-bnd}}

\begin{proof}
Consider the definition of $\widehat \btheta^\star$ and $\btheta^\star$, we can obtain
\begin{align*}
    \tilde \loss(\widehat \btheta^\star, \mathcal{D}) - \tilde \loss(\btheta^\star, \mathcal{D})
    & = 
    \tilde \loss(\widehat \btheta^\star, \mathcal{D}) - \widehat\loss(\widehat\btheta^\star) + \widehat\loss(\widehat \btheta^\star) - \widehat\loss(\btheta^\star)
    + \widehat \loss(\btheta^\star) - \tilde \loss(\btheta^\star, \mathcal{D}) \\
    & \leq
    2 \max_{\btheta \in \bar\bTheta \subset \bTheta} \left |
      \tilde \loss(\btheta, \mathcal{D}) - \widehat\loss(\btheta)
    \right|,
\end{align*}
where the inequality follows from the fact that $\widehat\loss(\widehat \btheta^\star) - \widehat\loss(\btheta^\star) \le 0$.
Moreover, observe that for any $\btheta\in \bar\bTheta \subset \bTheta$, by the definition of $\widehat \loss(\btheta)$ in \eqref{eq:def:global-loss}, we have
\begin{align*}
| \tilde \loss(\btheta, \mathcal{D}) - \widehat \loss(\btheta) | 
  & = 
  \left | \tilde \loss(\btheta, \mathcal{D}) -  \tsum_{i\in [p]} \alpha_i(\btheta) \cdot \widehat \loss_i(\btheta) \right | 
  \\
  & = 
  \left|\tilde \loss(\btheta, \mathcal{D}) -  \tsum_{i\in [p]} \alpha_i(\btheta) \cdot \tilde \loss(\btheta, \mathcal{D}_i) 
  + \tsum_{i\in[p]} \alpha_i(\btheta) \cdot \tilde \loss(\btheta, \mathcal{D}_i) - \tsum_{i\in[p]} \alpha_i(\btheta) \cdot \widehat \loss_i(\btheta) \right| 
  \\
  & \leq 
  \tsum_{i\in[p]} \alpha_i(\btheta) \left| \tilde \loss(\btheta, \mathcal{D}) - \tilde  \loss(\btheta, \mathcal{D}_i) \right| 
  + \tsum_{i\in[p]} \alpha_i(\btheta) \left| \tilde \loss(\btheta, \mathcal{D}_i) - \widehat \loss_i(\btheta) \right| 
  \\
  & \leq
  \tsum_{i\in[p]} \alpha_i(\btheta) \tilde \beta(\btheta) \mathcal W_1(\mathcal{D}, \mathcal{D}_i) + \tsum_{i\in [p]} \alpha_i(\btheta)\epsilon_i(\btheta, T),
\end{align*}
where the last inequality follows from assumption \eqref{eq:def:dist-lip-hat} and definition \eqref{eq:def:local-loss}.
\end{proof}

\subsection{Proof of Proposition~\ref{prop:bnd-wd}}

\begin{proof}
By the definition of 1-Wasserstein distance in \eqref{eq:def:1wd} and the fact that $\mathcal{D} = \sum_{i\in [p]} w_i \mathcal{D}_i$, we can obtain
\begin{align*}
    \mathcal W_1(\mathcal{D}, \mathcal{D}_i) 
    & =
    \sup_{f \in \mathsf F_1} 
    \E_{(x, y) \sim \mathcal{D}} f(x, y) - \E_{(x_i, y_i) \sim \mathcal{D}_i} f(x_i, y_i) \\
    & = 
    \sup_{f \in \mathsf F_1} 
    \tsum_{j\in [p]} w_j \E_{(x_j, y_j) \sim \mathcal{D}_j} f(x_j, y_j) - \E_{(x_i, y_i) \sim \mathcal{D}_i} f(x_i, y_i) \\
    & = 
    \sup_{f \in \mathsf F_1} 
    \tsum_{i \not= j, j\in [p]} w_j \left(\E_{(x_j, y_j) \sim \mathcal{D}_j} f(x_j, y_j) - \E_{(x_i, y_i) \sim \mathcal{D}_i} f(x_i, y_i) \right) \\
    & \leq 
    \tsum_{i \not= j, j\in [p]} w_j \left( 
    \sup_{f \in \mathsf F_1} 
    \E_{(x_j, y_j) \sim \mathcal{D}_j} f(x_j, y_j) - \E_{(x_i, y_i) \sim \mathcal{D}_i} f(x_i, y_i) \right) \\
    & \leq \tsum_{i \not= j, j\in [p]} w_j \mathcal W_1 (\mathcal{D}_j, \mathcal{D}_i).
\end{align*}
\end{proof}

\subsection{Proof of Proposition~\ref{prop:bnd-eps-i}}
\begin{proof}
By the definition of $\epsilon_i(\btheta, T)$ 
\begin{align*}
\epsilon_i(\btheta, T)
  & = \left| \tilde \loss(\btheta, \mathcal{D}_i) - \widehat \loss_i(\btheta) \right| \\
  & = \left|
      \underbrace{
      \tilde \loss(\btheta, \mathcal{D}_i) - \tilde \loss(\btheta_t^{(i)}, \mathcal{D}_i)
    }_{\text{Smoothness of } \tilde\loss}
    + \underbrace{
      \tilde \loss(\btheta_t^{(i)}, \mathcal{D}_i) - \widehat \loss_i(\btheta_t^{(i)})
    }_{\text{Modeling error}}
    + \underbrace{
      \widehat \loss_i(\btheta_t^{(i)}) - \widehat \loss_i(\btheta)
    }_{\text{Smoothness of } \widehat\loss}
  \right|, 
\end{align*}
where $\btheta_t^{(i)}, t \in [T]$ is {\em any} one of the HP tried during local HPO run on party $i \in [p]$. 

First note that 
\begin{align*}
|\tilde \loss(\btheta_t^{(i)}, \mathcal{D}_i) - \widehat \loss_i(\btheta_t^{(i)}) |
 & = 
 |\Loss_t^{(i)} - \widehat \loss_i(\btheta_t^{(i)}) |\\
 & \leq \max_t  |\Loss_t^{(i)} - \widehat \loss_i(\btheta_t^{(i)}) |\\
 & \leq \delta_i.
\end{align*}

In view of \eqref{eq:def:theta-lip-tilde} and \eqref{eq:def:dist-lip-hat}, we have
\begin{align*}
\epsilon_i(\btheta, T)
  & \leq 
   \tilde L(\mathcal{D}_i) d(\btheta, \btheta_t^{(i)})
  +  \delta_i
  + \widehat L_i d(\btheta, \btheta_t^{(i)}),
\end{align*}
which immediately implies the result in \eqref{eq:bnd-eps-i}.
\end{proof}

\subsection{Proposition~\ref{prop:bnd-eps-i} using modulus of continuity instead of Lipschitz continuity}\label{app:alt-proof-prop-bnd-eps-i}

\begin{proposition}\label{aprop:bnd-eps-i}
Assume that the estimated loss $\tilde \loss(\btheta, \mathcal{D}_i)$ and the loss surface $\loss_i (\btheta)$ admit concave functions $\tilde \omega_{\mathcal{D}_i}$ and $\widehat \omega_i$ respectively as a modulus of continuity with respect to $\btheta \in \bTheta$ for each party $i \in [p]$.  
Then, for any party $i, \ i\in[p]$, with the set of (HP, loss) pairs $\{(\btheta_t^{(i)}, \Loss_t^{(i)})\}_{t \in [T]}$ collected during the local HPO run for party $i$, for any $\btheta \in \bar{\bTheta} \subset \bTheta$, there exists $\tilde A_{\mathcal{D}_i}, \widehat A_i, \tilde B_{\mathcal{D}_i}, \widehat B_i \geq 0$ such that
\begin{equation}\label{eq:bnd-eps-i-2}
\epsilon_i(\btheta, T)\leq 
   \left( \tilde A_{\mathcal{D}_i} + \widehat A_i \right)
   \min_{t \in [T]} d(\btheta, \btheta_t^{(i)})
   + \tilde B_{\mathcal{D}_i} + \widehat B_i 
  + \delta_i,
\end{equation}
where $\delta_i = \max_t | \Loss_t^{(i)} - \widehat \loss_i (\btheta_t^{(i)}) |$
is the maximum per sample training error for the local loss surface $\widehat \loss_i$.
\end{proposition}

\begin{proof}
By the definition of $\epsilon_i(\btheta, T)$ 
\begin{align*}
\epsilon_i(\btheta, T)
  & = \left| \tilde \loss(\btheta, \mathcal{D}_i) - \widehat \loss_i(\btheta) \right| \\
  & = \left|
      \underbrace{
      \tilde \loss(\btheta, \mathcal{D}_i) - \tilde \loss(\btheta_t^{(i)}, \mathcal{D}_i)
    }_{\text{Smoothness of } \tilde\loss}
    + \underbrace{
      \tilde \loss(\btheta_t^{(i)}, \mathcal{D}_i) - \widehat \loss_i(\btheta_t^{(i)})
    }_{\text{Modeling error}}
    + \underbrace{
      \widehat \loss_i(\btheta_t^{(i)}) - \widehat \loss_i(\btheta)
    }_{\text{Smoothness of } \widehat\loss}
  \right|, 
\end{align*}
where $\btheta_t^{(i)}, t \in [T]$ is {\em any} one of the HP tried during local HPO run on party $i \in [p]$. 

First note that 
\begin{align*}
|\tilde \loss(\btheta_t^{(i)}, \mathcal{D}_i) - \widehat \loss_i(\btheta_t^{(i)}) |
 & = 
 |\Loss_t^{(i)} - \widehat \loss_i(\btheta_t^{(i)}) |\\
 & \leq \max_t  |\Loss_t^{(i)} - \widehat \loss_i(\btheta_t^{(i)}) |\\
 & \leq \delta_i.
\end{align*}

In view of \eqref{eq:est-loss-moc} and \eqref{eq:lsurf-moc}, we have
\begin{align*}
\epsilon_i(\btheta, T)
  & \leq 
   \tilde \omega_{\mathcal{D}_i} (d(\btheta, \btheta_t^{(i)}))
  +  \delta_i
  + \widehat\omega (d(\btheta, \btheta_t^{(i)})),
\\
  & \leq 
  \delta_i + 
  \min_{t \in [T]} \left(
   \tilde \omega_{\mathcal{D}_i} (d(\btheta, \btheta_t^{(i)}))
  + \widehat\omega (d(\btheta, \btheta_t^{(i)}))
  \right).
\end{align*}

Concavity of a function $\omega: [0, \infty] \to [0, \infty]$ implies that there exists $A, B > 0$ such that $\omega(t) \leq At + B$. Using that, we can find some $\tilde A_{\mathcal{D}_i}, \widehat A_i, \tilde B_{\mathcal{D}_i}, \widehat B_i > 0 $ which allows us to simplify the above to 
\begin{align*}
\epsilon_i(\btheta, T)
  & \leq 
  \delta_i + 
  (\tilde A_{\mathcal{D}_i} + \widehat A) \cdot \min_{t \in [T]}
  d(\btheta, \btheta_t^{(i)})
  + (\tilde B_{\mathcal{D}_i} + \widehat B).
\end{align*}
\end{proof}

\subsection{Relative regrets} \label{app:rr-bnd}
As a byproduct, we can also provide a bound for the following relative regret we use in our experiments.
\begin{corollary}\label{thm:rr-bnd}
Let us assume $\widehat \btheta^\star$ and $\btheta^\star$ are defined in \eqref{eq:def:flora-thetas} and \eqref{eq:def:global-theta}, and $\bar \btheta^\star$ and $\btheta_b$ are the hyper-parameter settings selected by centralized HPO and some baseline hyper-parameters, respectively, then we can bound the relative regret as follows, for a given data distribution $\mathcal{D}$, we have
\begin{align}\label{eq:rr-bnd}
   &
     \frac{\tilde \loss(\bar \btheta^\star, \mathcal{D}) - \tilde \loss(\widehat \btheta^\star, \mathcal{D})}{\tilde \loss(\bar \btheta^\star, \mathcal{D}) - \tilde  \loss(\btheta_b, \mathcal{D})}
   \nonumber \\
   & \quad
   \le 
     \frac{2 \max_{\btheta \in \bar\bTheta}\tsum_{i\in [p]} \alpha_i(\btheta)\left\{ 
     \tilde \beta(\btheta) 
     \tsum_{j \in [p], j \not= i} w_j \mathcal W_1 (\mathcal{D}_j, \mathcal{D}_i)
     + \left( \tilde L(\mathcal{D}_i) + \widehat L_i \right)
   \min_{t \in [T]} d(\btheta, \btheta_t^{(i)})
  + \delta_i\right\}}{\widehat \loss(\btheta_b, \mathcal{D}) - \widehat \loss(\bar \btheta^\star, \mathcal{D})}.
\end{align}
\end{corollary}
\begin{proof}
By the definition of relative regret, we have
\begin{align*}
    \frac{\tilde \loss(\bar \btheta^\star, \mathcal{D}) - \tilde \loss(\widehat \btheta^\star, \mathcal{D})}{\tilde \loss(\bar \btheta^\star, \mathcal{D}) - \tilde  \loss(\btheta_b, \mathcal{D})} 
    &= \frac{\tilde \loss (\widehat \btheta^\star, \mathcal{D})-\tilde \loss(\bar\btheta^\star, \mathcal{D})}{\tilde \loss(\btheta_b. \mathcal{D}) - \tilde \loss(\bar \btheta^\star, \mathcal{D})} \\
    & \le \frac{\tilde \loss (\widehat \btheta^\star,\mathcal{D})- \tilde \loss(\btheta^\star, \mathcal{D})}{\widehat \loss(\btheta_b, \mathcal{D}) - \widehat \loss(\bar \btheta^\star, \mathcal{D})}, 
\end{align*}
where the last inequality follows from the fact that $\btheta^\star$ is the minimizer of $\tilde \loss(\btheta, \mathcal{D})$. Moreover, in view of the result in Theorem~\ref{thm:og-bnd}, the result in \eqref{eq:rr-bnd} follows.
\end{proof}

\section{Experimental Setting}\label{app:exp}

\subsection{Dataset details} \label{asec:emp:datasets}

The details of the binary classification datasets used in our evaluation is reported in Table~\ref{tab:datasets}. We report the 10-fold cross-validated balanced accuracy of the default HP configuration on each of datasets with centralized training. The ``Gap'' column for the results for all datasets and models in \S \ref{asec:baseline-comp} denote the difference between the best 10-fold cross-validated balanced accuracy obtained via centralized HPO and the 10-fold cross-validated balanced accuracy of the default HP configuration. 

\begin{table}[t]
\centering
\caption{OpenML binary classification dataset details}
{\small
\begin{tabular}{llll} 
\toprule
 Data & rows & columns & class sizes \\ 
\midrule
 EEG eye state & 14980 & 14 & (8257, 6723) \\
 Electricity & 45312 & 8 & (26075, 19237) \\
 Heart statlog & 270 & 13 & (150, 120) \\
 Oil spill & 937 & 49 & (896, 41) \\
 Pollen & 3848 & 5 & (1924, 1924) \\
 Sonar & 208 & 61 & (111, 97) \\
 PC3 & 1563 & 37 & (1403, 160) \\
\bottomrule
\end{tabular}
}
\label{tab:datasets}
\end{table}

\subsection{Search space} \label{asec:emp:search-space}

We use the search space definition used in the NeurIPS 2020 Black-box optimization challenge (\url{https://bbochallenge.com/}), described in details in the API documentation\footnote{\url{https://github.com/rdturnermtl/bbo_challenge_starter_kit/\#configuration-space}}.

\subsubsection{Histogram based Gradient Boosted Trees}

Given this format for defining the HPO search space, we utilize the following precise search space for the {\tt HistGradientBoostingClassifier} in {\tt scikit-learn}:
{\footnotesize
\begin{alltt}
api_config = \{
    "max_iter": \{"type": "int", "space": "linear", "range": (10, 200)\},
    "learning_rate": \{"type": "real", "space": "log", "range": (1e-3, 1.0)\},
    "min_samples_leaf": \{"type": "int", "space": "linear", "range": (1, 40)\},
    "l2_regularization": \{"type": "real", "space": "log", "range": (1e-4, 1.0)\},
\}
\end{alltt}
}

The HP configuration we consider for the single-shot baseline described in \S \ref{sec:emp} is as follows:
{\footnotesize
\begin{alltt}
config = \{
    "max_iter": 100,
    "learning_rate": 0.1,
    "min_samples_leaf": 20,
    "l2_regularization": 0,
\}
\end{alltt}
}

\subsubsection{Kernel SVM with RBF kernel}

For  {\tt SVC(kernel="rbf")} in {\tt scikit-learn}, we use the following search space:

{\footnotesize
\begin{alltt}
api_config = \{
    "C": \{"type": "real", "space": "log", "range": (0.01, 1000.0)\},
    "gamma": \{"type": "real", "space": "log", "range": (1e-5, 10.0)\},
    "tol": \{"type": "real", "space": "log", "range": (1e-5, 1e-1)\},
\}
\end{alltt}
}

The single shot baseline we consider for {\tt SVC} from Auto-sklearn~\citep{feurer2015efficient} is:

{\footnotesize
\begin{alltt}
config = \{
    "C": 1.0,
    "gamma": 0.1,
    "tol": 1e-3,
\end{alltt}
}

\subsubsection{Multi-Layered Perceptrons}

For the {\tt MLPClassifier(solver="adam")} from {\tt scikit-learn}, we consider both architectural HP such as {\tt hidden-layer-sizes} as well as optimizer parameters such as {\tt alpha} and {\tt learning-rate-init} for the Adam optimizer~\citep{kingma2015adam}. We consider the following search space:

{\footnotesize
\begin{alltt}
api_config = \{
    "hidden_layer_sizes": \{"type": "int", "space": "linear", "range": (50, 200)\},
    "alpha": \{"type": "real", "space": "log", "range": (1e-5, 1e1)\},
    "learning_rate_init": \{"type": "real", "space": "log", "range": (1e-5, 1e-1)\},
\}
\end{alltt}
}

We utilize the following single shot baseline:

{\footnotesize
\begin{alltt}
config = \{
    "hidden_layer_sizes: 100,
    "alpha": 1e-4,
    "learning_rate_init": 1e-3,
\}
\end{alltt}
}

We fix the remaining HPs of {\tt MLPClassifier} as with values used by Auto-sklearn.

{\footnotesize
\begin{alltt}
activation="relu",
early_stopping=True,
shuffle=True,
batch_size="auto",
tol=1e-4,
validation_fraction=0.1,
beta_1=0.9,
beta_2=0.999,
epsilon=1e-8,
\end{alltt}
}

\subsection{Detailed results of comparison against baselines} \label{asec:baseline-comp}

Here we present the relevant details and the performance of \flora on the \flhpo of (i)~histograms based gradient boosted trees (HGB) in Table~\ref{atab:hgb}), (ii)~nonlinear support vector machines (SVM) in Table~\ref{atab:svm}, and (iii)~multi-layered perceptrons (MLP) in Table~\ref{atab:mlp-adam}. We use the search spaces and the single-shot baselines presented in \S \ref{asec:emp:search-space}. We utilize all 7 datasets for each of the method {\em except} for the Electricity dataset with SVM because of the infeasible amount of time taken by SVM on this dataset. For each setup, we report the following:
\begin{itemize}
\item Performance of the single-shot baseline (``SSBaseline''),
\item the best centralized HPO performance (``Best''),
\item the available ``Gap'' for improvement,
\item the minimum accuracy of the best local HP across all parties ``PMin'' $:= \min_{i \in [p]} \max_t (1 - \tilde \loss(\btheta_t^{(i)}, \mathcal{D}_i))$ 
\item the maximum accuracy of the best local HP across all parties ``PMax'' $:= \max_{i \in [p]} \max_t (1 - \tilde \loss(\btheta_t^{(i)}, \mathcal{D}_i))$ 
\item $\gamma_p = \nicefrac{\text{PMax}}{\text{PMin}}$, and finally
\item the regret for each of the considered loss surfaces in \flora.
\end{itemize}

For each of the three methods, we also report the aggregate performance over all considered datasets in terms of mean $\pm$ standard deviation (``mean$\pm$std''), inter-quartile range (``IQR''), Wins/Ties/Losses of \flora with respect to the single-shot baseline (``W/T/L''), and a one-sided Wilcoxon Signed Ranked Test of statistical significance (``WSRT'')  with the null hypothesis that the median of the difference between the single-shot baseline and \flora is positive against the alternative that the difference is negative (implying \flora improves over the baseline).' These aggregate metrics are collected in Table~\ref{atab:agg-all} along with a set of final aggregate metrics across all datasets and methods.

\begin{table}[htb]
\centering
\caption{HGB}
\label{atab:hgb}
{\small
\begin{tabular}{lccccc}
\toprule
Data & SSBaseline & Best & Gap & PMin & PMax \\
\midrule
PC3 & 58.99 & 63.81 & 4.82 & 61.67 & 64.37 \\ 
Pollen & 48.86 & 52.21 & 3.35 & 51.83 & 52.64 \\ 
Electricity & 87.75 & 92.84 & 5.10 & 88.42 & 89.19 \\
Sonar & 87.43 & 91.25 & 3.82 & 83.75 & 88.33  \\
Heart Statlog & 79.42 & 85.58 & 6.17 & 78.00 & 86.50 \\
Oil Spill & 63.22 & 74.58 & 11.36 & 68.16 & 82.16 \\
EEG Eye State & 89.96 & 94.66 & 4.70 & 91.80 & 92.29 \\
\bottomrule
\end{tabular}
\begin{tabular}{lccccc}
\toprule
Data & $\gamma_p$ & SGM & SGM+U & MPLM & APLM \\
\midrule
PC3 & 1.04 & 0.66 & 0.72 & 0.39 & 0.38 \\ 
Pollen & 1.02 & 0.43 & 0.54 & 0.43 & 0.69 \\ 
Electricity & 1.01 & 0.17 & 0.14 & 0.09 & 0.12 \\
Sonar & 1.05 & 1.33 & 0.41 & 0.92 & 0.71  \\
Heart Statlog & 1.11 & 0.69 & 0.55 & 0.89 & 0.50 \\
Oil Spill & 1.21 & 0.47 & 1.13 & 0.46 & 0.61 \\
EEG Eye State & 1.01 & 0.14 & 0.12 & 0.11 & 0.12 \\
\midrule
mean$\pm$std &    &   0.56 $\pm$ 0.37  &   0.52 $\pm$ 0.32  &   0.47 $\pm$ 0.31  &   0.45 $\pm$ 0.23 \\
IQR &  &   [0.30, 0.47, 0.68]  &   [0.27, 0.54, 0.64]  &   [0.25, 0.43, 0.67]  &   [0.25, 0.50, 0.65] \\
WTL &   &   6/0/1  &   6/0/1  &   7/0/0  &   7/0/0 \\
WSRT &  & (26, 0.02126)  &   (27, 0.01400)  &   (28, 0.00898)  &   (28, 0.00898) \\
\bottomrule
\end{tabular}
}
\end{table}

\paragraph{HGB.}
The results in Table~\ref{atab:hgb} indicate that, in almost all cases, with all loss functions, \flora is able to improve upon the baseline to varying degrees (there is only one case where SGM performs worse than the baseline on Sonar). On average (across the datasets), SGM+U, MPLM and APLM perform better than SGM as we expected. MPLM performs better than SGM both in terms of average and standard deviation. Looking at the individual datasets, we see that, for datasets with low $\gamma_p$ (EEG eye state, Electricity), all the proposed loss surface have low relative regret, indicating that the problem is easier as expected. For datasets with high $\gamma_p$ (Heart statlog, Oil spill), the relative regret of all loss surfaces are higher (but still much smaller than 1), indicating that our proposed single-shot scheme can show improvement even in cases where there is significant difference in the per-party losses (and hence datasets).

\begin{table}[htb]
\centering
\caption{SVM}
\label{atab:svm}
{\small
\begin{tabular}{lccccc}
\toprule
Data & SSBaseline & Best & Gap & PMin & PMax \\
\midrule
Pollen & 49.48 & 50.30 & \textcolor{red}{0.82} & 51.55 & 53.55 \\
Sonar & 80.20 & 89.29 & 9.09 & 83.33 & 87.92  \\
Heart Statlog & 83.67 & 84.92 & \textcolor{red}{1.25} & 77.00 & 88.00 \\
Oil Spill & 82.76 & 86.54 & 3.78 & 77.14 & 88.45 \\
EEG Eye State & 50.24 & 60.51 & 10.28 & 69.54 & 71.72 \\
PC3 & 74.03 & 77.96 & 3.92 & 75.26 & 76.95 \\
\bottomrule
\end{tabular}
\begin{tabular}{lccccc}
\toprule
Data & $\gamma_p$ & SGM & SGM+U & MPLM & APLM \\
\midrule
Pollen & 1.04 & \underline{1.35} & \underline{1.45} & \underline{2.84} & \underline{2.30} \\
Sonar  & 1.06 & 0.17 & 0.17 & 0.27 & 0.17  \\
Heart Statlog & \textcolor{red}{1.14} & 0.00 & 0.00 & \underline{6.80} & 0.67 \\
Oil Spill & \textcolor{red}{1.15} & \underline{1.28} & \underline{1.16} & \underline{1.12} & 0.41 \\
EEG Eye State & 1.03 & -0.01 & -0.01 & -0.02 & -0.01 \\
PC3 & 1.02 & 0.59 & 0.79 & 0.70 & 0.79 \\
\midrule
mean$\pm$std &  &  0.56 $\pm$ 0.57  &   0.59 $\pm$ 0.58  &   1.95 $\pm$ 2.35  &   0.72 $\pm$ 0.76 \\
IQR &  &  [0.04, 0.38, 1.11]  &   [0.04, 0.48, 1.07]  &   [0.38, 0.91, 2.41]  &   [0.23, 0.54, 0.76] \\
WTL &  &   4/0/2  &   4/0/2  &   3/0/3  &   5/0/1 \\
WSRT &  & (18, 0.05793)  &   (17, 0.08648)  &   (9, 0.62342)  &   (15, 0.17272) \\
\bottomrule
\end{tabular}
}
\end{table}

\paragraph{SVM.}
For SVM we continue with the datasets selected using HGB (datasets with a ``Gap'' of at least 3\%). Of the 7 datasets (Table~\ref{tab:datasets}), we skip Electricity because it takes a prohibitively long time for SVM to be trained on this dataset with a single HP. So we consider 6 datasets in this evaluation and present the corresponding results in Table~\ref{atab:svm}. Of the 6, note that 2 of these datasets (Pollen, Heart Statlog) have really small ``Gap'' (highlighted in \textcolor{red}{red} in Table~\ref{atab:svm}). Moreover, 2 of the datasets (Heart statlog, Oil Spill) also have really high $\gamma_p$ indicating a high level of heterogeneity between the per-party distributions (again highlighted in \textcolor{red}{red}). In this case, there are a couple of datasets (Oil Spill and Pollen) where \flora is unable to show any improvement over the single-shot baseline (see \underline{underlined} entries in Table~\ref{atab:svm}), but both these cases either have a small or moderate ``Gap'' and/or have a high $\gamma_p$. Moreover, in one case, MPLM incurs a regret of 6.8, but this is a case with really high $\gamma_p = 1.14$ -- MPLM rejects any HP that has a low score in even one of the parties, and in that process reject all promising HPs since the local HPOs on these disparate distributions did not concentrate on the same region of the HP space, thereby incuring a high MPLM loss in almost all regions of the HP where some local HPO focused on. Other than these expected hard cases, \flora is able to improve upon the baseline in most cases, and achieve optimal performance (zero regret) in a few cases (EEG Eye State, Heart Statlog).

\begin{table}[htb]
\centering
\caption{MLP-Adam}
\label{atab:mlp-adam}
{\small
\begin{tabular}{lccccc}
\toprule
Data & SSBaseline & Best & Gap & PMin & PMax \\
\midrule
Pollen & 50.39 & 51.26 & \textcolor{red}{0.87} & 51.46 & 52.23 \\
Electricity & 76.95 & 78.06 & \textcolor{red}{1.11} & 77.01 & 77.39 \\
Sonar & 61.63 & 79.32 & 17.69 & 69.17 & 78.75 \\
Heart Statlog & 72.17 & 85.17 & 13.00 & 79.50 & 89.50 \\
Oil Spill & 50.00 & 65.22 & 15.22 & 54.83 & 63.63 \\
EEG Eye State & 49.99 & 51.66 & \textcolor{red}{1.67} & 50.02 & 51.84 \\
PC3 & 50.00 & 59.56 & 9.56 & 53.47 & 56.60 \\
\bottomrule
\end{tabular}
\begin{tabular}{lccccc}
\toprule
Data & $\gamma_p$ & SGM & SGM+U & MPLM & APLM \\
\midrule
Pollen & 1.02 & \underline{1.88} & \underline{1.45} & \underline{1.45} & \underline{1.31} \\
Electricity & 1.00 & 0.24 & 0.41 & 0.16 & 0.53 \\
Sonar & \textcolor{red}{1.14} & 0.26 & 0.55 & 0.52 & 0.39 \\
Heart Statlog & \textcolor{red}{1.13} & 0.46 & 0.37 & 0.42 & 0.28 \\
Oil Spill & \textcolor{red}{1.16} & 0.80 & 1.03 & 1.00 & 0.79 \\
EEG Eye State & 1.04 & \underline{0.99} & \underline{0.99} & \underline{0.99} & \underline{0.99} \\
PC3  & 1.06 & 0.96 & 1.00 & 0.89 & 0.90 \\
\midrule
mean$\pm$std &    &   0.80 $\pm$ 0.53  &   0.83 $\pm$ 0.37  &   0.78 $\pm$ 0.40  &   0.74 $\pm$ 0.34 \\
IQR &    &   [0.36, 0.80, 0.97]  &   [0.48, 0.99, 1.01]  &   [0.47, 0.89, 1.00]  &   [0.46, 0.79, 0.95] \\
WTL &    &   6/0/1  &   4/1/2  &   5/1/1  &   6/0/1 \\
WSRT &    &   (21, 0.11836)  &   (15, 0.17272)  &   (18, 0.05793)  &   (24, 0.04548) \\
\bottomrule
\end{tabular}
}
\end{table}

\paragraph{MLP.}
We consider all 7 datasets for the evaluation of \flora on \flhpo for MLP HPs and present the results in Table~\ref{atab:mlp-adam}. As with SVM, there are a few datasets with a small room for improvement (``Gap'') and/or high $\gamma_p$, again highlighted in \textcolor{red}{red} in Table~\ref{atab:mlp-adam}. In some of these cases, \flora is unable to improve upon the single-shot baseline (Pollen, EEG Eye State). Other than these hard cases, \flora again able to show significant improvement over the single-shot baseline, with APLM performing the best.

\begin{table}[!!hb]
\centering
\caption{Aggregate Table}
\label{atab:agg-all}
{\footnotesize
\begin{tabular}{lcccccc}
\toprule
Agg. & Method & SGM & SGM+U & MPLM & APLM \\
\midrule
mean $\pm$ std. &  HGB  &   0.56 $\pm$ 0.37  &   0.52 $\pm$ 0.32  &   0.47 $\pm$ 0.31  &   0.45 $\pm$ 0.23 \\
&  SVM  &   0.56 $\pm$ 0.57  &   0.59 $\pm$ 0.58  &   1.95 $\pm$ 2.35  &   0.72 $\pm$ 0.76 \\
&  MLP &   0.80 $\pm$ 0.53  &   0.83 $\pm$ 0.37  &   0.78 $\pm$ 0.40  &   0.74 $\pm$ 0.34 \\
& Overall & {\bf 0.64 $\pm$ 0.51} & {\bf 0.64 $\pm$ 0.51} & 1.02 $\pm$ 1.46 & {\bf 0.63 $\pm$ 0.50} \\
\midrule
IQR &  HGB  &   [0.30, 0.47, 0.68]  &   [0.27, 0.54, 0.64]  &   [0.25, 0.43, 0.67]  &   [0.25, 0.50, 0.65] \\
&  SVM  &   [0.04, 0.38, 1.11]  &   [0.04, 0.48, 1.07]  &   [0.38, 0.91, 2.41]  &   [0.23, 0.54, 0.76] \\
&  MLP  &   [0.36, 0.80, 0.97]  &   [0.48, 0.99, 1.01]  &   [0.47, 0.89, 1.00]  &   [0.46, 0.79, 0.95] \\
& Overall & [{\bf 0.22}, {\bf 0.53}, 0.97] & [0.32, 0.55, 1.01] & [0.36, 0.61, 0.99] & [0.36, 0.57, {\bf 0.79}] \\
\midrule
W/T/L & HGB &  6/0/1  &   6/0/1  &   7/0/0  &   7/0/0 \\
& SVM &  4/0/2  &   4/0/2  &   3/0/3  &   5/0/1 \\
& MLP &   6/0/1  &   4/1/2  &   5/1/1  &   6/0/1 \\
& Overall & 16/0/4 & 14/1/5 & 15/1/4 & {\bf 18/0/2} \\
\midrule
WSRT 1 sided &  HGB & (26, 0.02126)  &   (27, 0.01400)  &   (28, 0.00898)  &   (28, 0.00898) \\
& SVM  &   (18, 0.05793)  &   (17, 0.08648)  &   (9, 0.62342)  &   (15, 0.17272) \\
& MLP  &  (21, 0.11836)  &   (15, 0.17272)  &   (18, 0.05793)  &   (24, 0.04548) \\
& Overall & (174, 0.00499) & (164, 0.00272) & (141, 0.03206) & {\bf (183.5, 0.00169)}
 \\
\bottomrule
\end{tabular}
}
\end{table}

\paragraph{Aggregate.}
The results for all the methods and datasets are aggregated in Table~\ref{atab:agg-all}. All \flora loss surfaces show strong performance with respect to the single-shot baseline, with significantly more wins than losses, and 3rd-quartile regret values less than 1 (indicating improvement over the baseline). All \flora loss surfaces have a p-value of less than $0.05$, indicating that we can reject the null hypothesis. 
Overall, APLM shows the best performance over all loss surfaces, both in terms of Wins/Ties/Losses over the baseline as well as in terms of the Wilcoxon Signed Rank Test, with the highest statistic and a p-value close to $10^{-3}$. APLM also has significantly lower 3rd-quartile than all other loss surfaces. MPLM appears to have the worst performance but much of that is attributable to the really high regret of 6.8 and 2.84 it received for SVM with Heart Statlog and Pollen (both hard cases as discussed earlier). Otherwise, MPLM performs second best both for \flhpo with HGB and MLP.

\subsection{Effect of the number of local HPO rounds per party} \label{asec:emp:Tdep}

In this experiment, we report additional results  to study the effect of the ``thoroughness'' of the local HPO runs (in terms of the number of HPO rounds $T$) on the overall performance of \flora for all the loss surfaces in Table~\ref{atab:local-hpo-rounds}. In almost all cases, \flora does not require $T$ to be too large to get enough information about the local HPO loss surface to get to its best possible performance.

\begin{table}[t]
\centering
\caption{Effect of $T$.} 
\label{atab:local-hpo-rounds}
{\small
\begin{tabular}{llcccccc}
\toprule
Method & data          & $T$ &$\gamma_p$& SGM  & SGM+U & MPLM & APLM \\
\midrule
MLP    & Heart Statlog & 5   & 1.13 & 0.58 & 0.33  & 0.22 & 0.56 \\
       &               & 10  & 1.13 & 0.33 & 0.16  & 0.60 & 0.39 \\
       &               & 20  & 1.13 & 0.49 & 0.15  & 0.24 & 0.44 \\
       &               & 40  & 1.13 & 0.44 & 0.30  & 0.42 & 0.29 \\
       &               & 60  & 1.13 & 0.37 & 0.15  & 0.33 & 0.22 \\
       &               & 80  & 1.13 & 0.35 & 0.40  & 0.35 & 0.26 \\
\midrule
MLP    & Sonar         & 5   & 1.14 & 0.38 & 0.38  & 0.51 & 0.78 \\
       &               & 10  & 1.14 & 0.45 & 0.23  & 0.43 & 0.62 \\
       &               & 20  & 1.14 & 0.39 & 0.24  & 0.36 & 0.30 \\
       &               & 40  & 1.14 & 0.23 & 0.37  & 0.65 & 0.49 \\
       &               & 60  & 1.14 & 0.53 & 0.14  & 0.34 & 0.48 \\
       &               & 80  & 1.14 & 0.46 & 0.07  & 0.19 & 0.30 \\
\midrule
SVM    & Sonar         & 5   & 1.06 & 0.17 & 0.17  & 1.16 & 0.28 \\
       &               & 10  & 1.06 & 0.17 & 0.43  & 0.34 & 0.27 \\
       &               & 20  & 1.06 & 0.17 & 0.17  & 0.17 & 0.17 \\
       &               & 40  & 1.06 & 0.17 & 0.17  & 0.22 & 0.27 \\
       &               & 60  & 1.06 & 0.17 & 0.17  & 0.17 & 0.17 \\
       &               & 80  & 1.06 & 0.17 & 0.11  & 0.27 & 0.17 \\
\midrule
SVM    & EEG           & 5   & 1.03 & -0.01 &-0.01 & 0.92 & 0.16 \\
       &               & 10  & 1.03 & -0.01 &-0.01 & -0.01& 0.14 \\
       &               & 20  & 1.03 & -0.01 &-0.01 & -0.00& -0.00 \\
       &               & 40  & 1.03 & -0.02 &-0.02 & 0.01 & 0.00 \\
       &               & 60  & 1.03 & -0.01 &-0.01 & 0.02 & 0.03 \\
       &               & 80  & 1.03 & -0.01 &-0.01 & 0.01 & -0.0` \\
\bottomrule
\end{tabular}
}
\end{table}

\subsection{Effect of communication overhead} \label{asec:emp:Tcommdep}

While in the previous experiment, we studied the effect of the thoroughness of the local HPO runs on the performance of \flora, here we consider a subtly different setup. We assume that each party performs $T=100$ rounds of local asynchronous HPO. However, instead of sending all $T$ (HP, loss) pairs, we consider sending $T' < T$ of the ``best'' (HP, loss) pairs -- that is, (HP, loss) pairs with the $T'$ lowest losses. Changing the value of $T'$ trades off the communication overhead of the \flora step where the aggregators collect the per-party loss pairs (Algorithm~\ref{alg:fl-hpo-lsa}, line 5). We consider 2 datasets each for 2 of the methods (SVM, MLP) and all the loss surfaces for \flora, and report all the results in Table~\ref{atab:comm-overhead}.

\begin{table}[htb]
\centering
\caption{Effect of the number of best (HP, loss) pairs $T' < T$ sent to aggregator by each party after doing local HPO with $T = 100$.} 
\label{atab:comm-overhead}
{\small
\begin{tabular}{llcccccc}
\toprule
Method & data          & $T' < T$ &$\gamma_p$& SGM  & SGM+U & MPLM & APLM \\
\midrule
MLP    & Heart Statlog & 5   &  1.13 & 0.33 & 0.27 & 0.38 & 0.71 \\
       &               & 10  &  1.13 & 0.35 & 0.31 & 0.33 & 1.72 \\
       &               & 20  &  1.13 & 0.42 & 0.39 & 2.02 & 0.55 \\
       &               & 40  &  1.13 & 0.34 & 0.44 & 0.88 & 0.51 \\
       &               & 60  &  1.13 & 0.38 & 0.22 & 0.31 & 0.32 \\
       &               & 80  &  1.13 & 0.34 & 0.38 & 0.22 & 0.33 \\
\midrule
MLP    & Sonar         & 5   &  1.14 & 0.39 & 0.50 & 1.78 & 0.65 \\
       &               & 10  &  1.14 & 0.73 & 0.18 & 1.66 & 0.58 \\
       &               & 20  &  1.14 & 0.20 & 0.41 & 1.23 & 0.37 \\
       &               & 40  &  1.14 & 0.60 & 0.42 & 0.18 & 0.51 \\
       &               & 60  &  1.14 & 0.10 & 0.33 & 0.55 & 0.26 \\
       &               & 80  &  1.14 & 0.47 & 0.41 & 0.34 & 0.32 \\
\midrule
SVM    & EEG Eye State & 5   &  1.03 & -0.02 & -0.01 & 0.39 & 1.02 \\
       &               & 10  &  1.03 & -0.01 & -0.01 & 1.02 & 1.02 \\
       &               & 20  &  1.03 & -0.01 & -0.01 & 0.01 & -0.01 \\
       &               & 40  &  1.03 & -0.01 & -0.01 & -0.00 & -0.01 \\
       &               & 60  &  1.03 & -0.01 & -0.01 & -0.01 & -0.01 \\
       &               & 80  &  1.03 & -0.01 & -0.01 & -0.01 & -0.01 \\
\midrule
SVM    & Sonar         & 5   &  1.06 & 0.17 & 0.17 & 0.43 & 1.43 \\
       &               & 10  &  1.06 & 0.17 & 0.17 & 0.17 & 0.17 \\
       &               & 20  &  1.06 & 0.17 & 0.17 & 0.22 & 0.17 \\
       &               & 40  &  1.06 & 0.17 & 0.38 & 0.27 & 0.17 \\
       &               & 60  &  1.06 & 0.17 & 0.43 & 0.27 & 0.17 \\
       &               & 80  &  1.06 & 0.17 & 0.43 & 0.27 & 0.17 \\
\bottomrule
\end{tabular}
}
\end{table}


\end{document}